%% file: Arxiv-BL.tex
\documentclass[11pt]{article}
\usepackage{fullpage}

\input{math-commands.tex}

\usepackage{microtype}
\usepackage{graphicx}
\usepackage{subfigure}
\usepackage{booktabs} 
\usepackage{makecell}
\usepackage{multirow}
\usepackage{bm}
\usepackage{amsthm}
\usepackage{hyperref}
\usepackage{url}
\usepackage{multicol}
\usepackage[utf8]{inputenc} 
\usepackage[T1]{fontenc}    
\usepackage{hyperref}
\usepackage{float}
\usepackage{caption}
\usepackage{epstopdf}
\usepackage{etoc}
\usepackage[sort,round]{natbib}			

\etocdepthtag.toc{mtchapter}
\etocsettagdepth{mtchapter}{subsection}
\etocsettagdepth{mtappendix}{none}

\hypersetup{
    pdfnewwindow=true,     
    colorlinks=true,     
    linkcolor=black,     
    citecolor=black,     
    filecolor=magenta,     
    urlcolor=cyan     
}
\usepackage{url}            
\usepackage{booktabs}       
\usepackage{amsmath,amsthm,amssymb,amsfonts}
\usepackage{nicefrac}       
\usepackage{microtype}      
\usepackage{xcolor}         

\usepackage{makecell}
\usepackage{pgfplots}\pgfplotsset{compat=1.18}
\usepackage{cancel}
\usepackage{mathtools}
\usepackage{bbm,dsfont}
\usepackage{enumitem}
\usepackage{xspace}
\usepackage[capitalize]{cleveref}
\usepackage{semantic}
\usepackage{bm}
\usepackage{mathtools}
\usepackage{pifont}
\allowdisplaybreaks

\usepackage{hyperref}

\makeatletter
\newtheorem*{rep@theorem}{\rep@title}
\newcommand{\newreptheorem}[2]{%
\newenvironment{rep#1}[1]{%
 \def\rep@title{#2 \ref{##1}}%
 \begin{rep@theorem}}%
 {\end{rep@theorem}}}
\makeatother

\theoremstyle{plain}
\newtheorem{theorem}{Theorem}[section]
\newtheorem{proposition}[theorem]{Proposition}
\newtheorem{lemma}[theorem]{Lemma}
\newtheorem{corollary}[theorem]{Corollary}
\theoremstyle{definition}

\theoremstyle{remark}
\newtheorem{remark}[theorem]{Remark}

\newreptheorem{theorem}{Theorem}
\newreptheorem{lemma}{Lemma}
\newreptheorem{proposition}{Proposition}
\newreptheorem{corollary}{Corollary}
\newcommand{\nn}{\nonumber}
\newcommand{\purple}{\color{black}}
\newcommand{\blue}{\color{black}}

\newcommand{\KLr}{\mathrm{KL}}

\usepackage{algorithm}
\usepackage{algorithmic}

\usepackage{xspace}
\usepackage[colorinlistoftodos, color=blue!30!white 
]{todonotes}                                        
\title{Semi-supervised Batch Learning From Logged Data}

\author{
  Gholamali Aminian$^{ *,\,1}$
 \and  Armin Behnamnia$^{*,\,2}$
 \and
 Roberto Vega$^{3}$
 \and
 Laura Toni$^{4}$
 \and 
 Chengchun Shi$^{5}$
 \and 
 Hamid R. Rabiee$^{2}$
 \and 
 Omar Rivasplata$^{6}$
 \and
 Miguel R. D. Rodrigues$^{4}$
}

\begin{document}

\maketitle
\renewcommand\thefootnote{$*$}\footnotetext{Equal contribution.}
\renewcommand\thefootnote{$^1$}\footnotetext{The Alan Turing Institute.}
\renewcommand\thefootnote{$^2$}\footnotetext{Department of Computer Engineering, Sharif University of Technology}

\renewcommand\thefootnote{$^3$}\footnotetext{Exo Imaging Company}
\renewcommand\thefootnote{$^4$}\footnotetext{Department of Electronic and Electrical Engineering, University College London }
\renewcommand\thefootnote{$^5$}\footnotetext{Department of Statistics, London School of Economics and Political Science }
\renewcommand\thefootnote{$^6$}\footnotetext{Department of Statistics, University College London}


\begin{abstract}
Off-policy learning methods are intended to learn a policy from logged data, which includes context, action, and feedback (cost or reward) for each sample point. In this work we build on the counterfactual risk minimization framework, which also assumes access to propensity scores. We propose learning methods for problems where feedback is missing for some samples, so there are samples with feedback and samples missing-feedback in the logged data. We refer to this type of learning as semi-supervised batch learning from logged data, which arises in a wide range of application domains. We derive a novel upper bound for the true risk under the inverse propensity score estimator to address this kind of learning problem. Using this bound, we propose a regularized semi-supervised batch learning method with logged data where the regularization term is feedback-independent and, as a result, can be evaluated using the logged missing-feedback data. Consequently, even though feedback is only present for some samples, a learning policy can be learned by leveraging the missing-feedback samples. The results of experiments derived from benchmark datasets indicate that these algorithms achieve policies with better performance in comparison with logging policies.
\end{abstract}

\section{Introduction}\label{sec:introduction}
Off-policy learning from logged data is an important problem in reinforcement learning theory and practice. 
The logged `known-feedback' dataset represents interaction logs of a system with its environment; recording context, action, propensity score (i.e., probability of the action selection for a given context under the logging policy), and feedback. 
The literature has considered this setting concerning contextual bandits and partially labeled observations. It is used in many real applications, e.g., recommendation systems~\citep{aggarwal2016recommender, li2011unbiased}, personalized medical treatments~\citep{kosorok2019precision,bertsimas2017personalized} and personalized advertising campaigns~\citep{tang2013automatic,bottou2013counterfactual}.
However, there are two main obstacles to learning from this kind of logged data: first, the observed feedback is available for the chosen action only; and second, the logged data is taken under the logging policy so that it could be biased. Batch learning with logged bandit feedback, also known as Counterfactual Risk Minimization (CRM), is a strategy for off-policy learning from logged `known-feedback' datasets, which has been proposed by \citet{swaminathan2015batch} to tackle these challenges.

 Batch learning with logged bandit feedback has led to promising results in some settings, including advertising and recommendation systems. 
However, there are some scenarios where the logged dataset is generated in an uncontrolled manner, posing significant obstacles such as unobserved feedback for some chosen context and action pairs. For example, consider an advertising system server where some ads (actions) are shown to different clients (contexts) according to a conditional probability (propensity score). Now, suppose that the connections between the clients and the server are corrupted momentarily such that the server does not receive any feedback, i.e., whether or not the user has clicked on some ads. Under this scenario, we have access to `missing-feedback' data indicating the chosen clients, the shown ads, the probability of shown ads but missing feedback, and some logged data containing feedback. 
Likewise, there are other scenarios where obtaining feedback samples for some context and action (and propensity score) samples may be challenging since it might be expensive or unethical, such as in finance \citep{musto2015personalized} or healthcare \citep{chakrabortty2018efficient}. 


We call Semi-supervised Batch Learning (S2BL) our approach to learning in these scenarios, where we have access to the logged missing-feedback dataset, besides the logged known-feedback dataset, which was the typical data considered in previous approaches. 

This paper proposes algorithms that leverage the logged missing-feedback and known-feedback datasets in an off-policy optimization problem. The contributions of our work are as follows:
\begin{itemize}
[leftmargin=15pt,topsep=1.5pt,itemsep=1.5pt]

    \item We propose a novel upper bound on the true risk of a policy, in terms of the truncated inverse propensity score (IPS) estimator and divergences (KL and reverse KL) between the logging policy and a learning policy.

    \item Inspired by this upper bound, we propose regularization approaches based on KL divergence or reverse KL divergence between the logging policy and a learning policy, which are independent of feedback and hence can be optimized using the logged missing-feedback dataset. We also propose consistent and asymptotically unbiased estimators of KL divergence and reverse KL divergence between the logging policy and a learning policy.

    \item We report on experiments conducted on various datasets to assess the effectiveness of our proposed algorithms. The results demonstrate our method's ability to leverage logged missing-feedback data across different setups, encompassing both linear and deep structures. Furthermore, we offer a comparative analysis against established baselines in the literature.
\end{itemize}

\section{Related Works}

Various methods have been developed to learn from logged known-feedback datasets. The main approach is batch learning with a logged known-feedback dataset (bandit feedback), discussed next. Appendix (App.)~\ref{App: related works} discusses other related topics and the corresponding literature. 

\textbf{Batch Learning with Logged known-feedback dataset:} The mainstream approach for off-policy learning from a logged known-feedback dataset is CRM~ \citep{swaminathan2015batch}. In particular, \citet{joachims2018deep} proposed a new approach to train a neural network, where the output of the softmax layer is considered as the policy, and the network is trained using the available logged known-feedback dataset. 
Our work builds on the former, albeit proposing methods to learn from logged missing-feedback data besides the logged known-feedback dataset.
CRM has also been combined with domain adversarial networks by \citet{atan2018counterfactual}.  \citet{wu2018variance} proposed a new framework for CRM based on regularization by Chi-square divergence between the learning policy and the logging policy, and a generative-adversarial approach is proposed to minimize the regularized empirical risk using the logged known-feedback dataset. \citet{xie2018off} introduced the surrogate policy method in CRM. The combination of causal inference and counterfactual learning was studied by \citet{bottou2013counterfactual}. Distributional robust optimization is applied in CRM by \citet{faury2020distributionally}. A lower bound on the expected reward in CRM under Self-normalized Importance Weighting was derived by \citet{kuzborskij2021confident}. The sequential CRM where the logged known-feedback dataset is collected at each iteration of training is studied by \citet{zenati2023sequential}. In this work, we introduce a novel algorithm that leverages both the logged missing-feedback dataset and the logged known-feedback dataset.

{\blue \textbf{Pessimism Method and Off-policy Reinforcement Learning:} The pessimism concept originally, introduced in offline reinforcement learning \citep{buckman2020importance,jin2021pessimism}, aims to derive an optimal policy within Markov decision processes (MDPs) by utilizing pre-existing datasets \citep{rashidinejad2022optimal,rashidinejad2021bridging,yin2021towards,yan2023efficacy}. This concept has also been adapted to contextual bandits, viewed as a specific MDP instance. Recently, a ‘design-based’ version of the pessimism principle is proposed by \citet{jin2022policy} who propose a data-dependent and policy-dependent regularization inspired by a lower confidence bound (LCB) on the estimation uncertainty of the augmented-inverse-propensity-weighted (AIPW)-type estimators which also includes IPS estimators. Our work differs from that of \citet{jin2022policy} as our regularization is inspired by variance reduction of truncated IPS estimator. However, the regularization used by \citet{jin2022policy} is motivated by a LCB. In addition, our regularization, can be implemented by deep neural networks. }

\textbf{Importance Weighting:} This method has been proposed for off-policy estimation and learning \citep{thomas2015high,swaminathan2015batch}. 
Due to its large variance in many cases \citep{rosenbaum1983central}, some truncated importance sampling methods are proposed, including the IPS estimator with a truncated ratio of policy and logging policy \citep{ionides2008truncated}, IPS estimator with truncated propensity score \citep{strehl2010learning} or self-normalizing estimator \citep{swaminathan2015self}.  
A balance-based weighting approach for policy learning, which outperforms other estimators, was proposed by~\citet{kallus2018balanced}. 
A generalization of importance sampling by considering samples from different policies is studied by \citet{papini2019optimistic}. The weights can be estimated directly by sampling from contexts and actions using Direct Importance Estimation \citep{sugiyama2007direct}. A convex surrogate for the regularized true risk by the entropy of learning policy is proposed by \citet{chen2019surrogate}. An exponential smoothed version of the IPS estimator is proposed by \citet{aouali2023exponential}. Other corrections of IPS estimator are also proposed by \citet{metelli2021subgaussian,su2020doubly}. IX-estimator~\citep{neu2015explore} where a constant is added to logging policy is studied by \citet{gabbianelli2023importance}.
This work considers the IPS estimator based on a truncated propensity score.

\section{Preliminaries}\label{Sec: Preliminaries}
\textbf{Notations:} We adopt the following convention for random variables and their distributions in the sequel. 
A random variable is denoted by an upper-case letter (e.g., $Z$), an arbitrary value of this variable is denoted with the lower-case letter (e.g., $z$), and its space of all possible values with the corresponding calligraphic letter (e.g., $\mathcal{Z}$). 
This way, we can describe generic events like $\{ Z = z \}$ for any $z \in \mathcal{Z}$, or events like $\{ g(Z) \leq 5 \}$ for functions $g : \mathcal{Z} \to \mathbb{R}$. 
The probability distribution of the random variable $Z$ is denoted $ P_Z$. 
The joint distribution of a pair of random variables $(Z_1,Z_2)$ is denoted by $P_{Z_1,Z_2}$. We denote the set of integer numbers from 1 to $n$ by $[n]\triangleq \{1,\cdots,n\}$.

\textbf{Divergence Measures:} If  $P$ and $Q$ are probability measures over $\mathcal{Z}$, 
the Kullback-Leibler (KL) divergence $\KLr(P\|Q)$ is given by
$\KLr(P\|Q)\triangleq\int_{\mathcal{Z}}\log\bigl(\frac{dP}{dQ}\bigr) dP$ when $P$ is absolutely continuous\footnote{$P$ is absolutely continuous with respect to $Q$ if $P(A) = 0$ whenever $Q(A) = 0$, for measurable $A \subset \mathcal{Z}$.}  with respect to $Q$, and $\KLr(P\|Q)\triangleq\infty$ otherwise. 

The so-called `reverse KL divergence' is $\KLr(Q\|P)$,
with arguments in the reverse order. 
The chi-square divergence is  $\chi^2(P\|Q)\triangleq\int_\mathcal{Z} (\frac{dP}{dQ})^2dQ-1$.

For a pair of random variables $(T,Z)$, 
the conditional KL divergence $\KLr(P_{T|Z} \| Q_{T|Z})$ is defined as 
$$
\KLr(P_{T|Z} \| Q_{T|Z})\triangleq\int_\mathcal{Z}\KLr(P_{T|Z=z} \| Q_{T|Z=z}) dP_{Z}(z).
$$
The conditional chi-square divergence $\chi^2(P_{T|Z} \| Q_{T|Z})$ is defined similarly. 

\textbf{Problem Formulation:} Let $\mathcal{X}$ be the set of contexts and $\mathcal{A}$ the finite set of actions, with $\vert\mathcal{A}\vert=k\ge 2$. We consider policies as conditional distributions over actions, given contexts. For each pair of context and action $(x,a)\in \mathcal{X}\times \mathcal{A}$ and policy $\pi \in \Pi$, where $\Pi$ is the set of policies, the value $\pi(a|x)$ is defined as the conditional probability of choosing action $a$ given context $x$ under the policy $\pi$. 

Inspired by \citet{swaminathan2015batch}, a cost\footnote{The cost can be viewed as the opposite (negative) of the reward. Consequently, a low cost (equivalent to maximum reward) signifies user (context) satisfaction with the given action, and conversely.} function $c:\mathcal{X}\times \mathcal{A}\rightarrow [-1,0]$, which is unknown, defines the cost of each observed pair of context and action. However, in a \emph{logged known-feedback} setting, we only observe the feedback for the chosen action $a$ in a given context $x$, under the logging policy $\pi_0(a|x)$. We have access to the logged known-feedback dataset $S=(x_i,a_i,p_i,c_i)_{i=1}^n$ where each `data point' $(x_i,a_i,p_i,c_i)$ contains the context $x_i$ which is sampled from unknown distribution $P_X$, the action $a_i$ which is sampled from the logging policy $\pi_0(\cdot|x_i)$, the propensity score $p_i\triangleq\pi_0(a_i|x_i)$, and the observed cost $c_i \triangleq r(x_i,a_i)$ under logging policy $\pi_0(a_i|x_i)$.

The \emph{true risk} of a policy $\pi_\theta$ is,
\begin{align}
    \label{Eq: True Risk}
    R(\pi_\theta)&=\mathbb{E}_{P_{X}}[\mathbb{E}_{\pi_{\theta}(A|X)}[c(A,X)]].
\end{align}
Our objective is to find an optimal $\pi_\theta^\star$ which minimizes $ R(\pi_\theta)$, i.e.,
$
\pi_\theta^\star=\operatorname*{arg\,min}_{\pi_\theta\in \Pi_\theta}R(\pi_\theta),
$
where $\Pi_\theta$ is the set of all policies parameterized by $\theta\in\Theta$.
We denote the importance weighted cost function as $w(A,X)c(A,X)$, where \[w(A,X)=\frac{\pi_\theta(A|X)}{\pi_0(A|X)}.\]

As discussed by \citet{swaminathan2015self}, see also \citet{rosenbaum1983central}; we can apply the IPS estimator over logged known-feedback dataset $S$ to get an unbiased estimator of the risk (an \emph{empirical risk})  by considering the importance weighted cost function as,
\begin{align}
    \label{Eq: emp risk}
    \hat{R}(\pi_\theta,S)&=\frac{1}{n}\sum_{i=1}^n c_i w(a_i,x_i),
\end{align}
where $w(a_i,x_i)=\frac{\pi_\theta(a_i|x_i)}{\pi_0(a_i|x_i)}$. 
 The IPS estimator as an unbiased estimator has bounded variance if the $\pi_\theta(A|X)$ is absolutely continuous  with respect to $\pi_0(A|X)$, cf. \citet{strehl2010learning,langford2008exploration}. 
 %
 For the issue of the large variance of the IPS estimator, many estimators are proposed ~\citep{strehl2010learning,ionides2008truncated,swaminathan2015self}, e.g., truncated IPS estimator. In this work we consider truncated IPS estimator with threshold $\nu\in(0,1]$ as follows:
\begin{align}
    \label{Eq: truncated emp risk}
    \hat{R}_\nu(\pi_\theta,S)&=\frac{1}{n}\sum_{i=1}^n c_i w_\nu(a_i,x_i),
\end{align}
 where $w_\nu(a_i,x_i)=\frac{\pi_\theta(a_i,x_i)}{\max(\nu,\pi_0(a_i,x_i))}$. Note that the truncation threshold $\nu\in(0,1]$ implies an upper bound on the importance weights, $\sup_{(x,a)\in \mathcal{X}\times \mathcal{A}} w_\nu(a,x) \leq \nu^{-1}$.

{\purple In our S2BL setting, besides the logged known-feedback dataset $S$ we also have access to a missing-feedback dataset $S_{u}=(x_j,a_j,p_j)_{j=1}^m$. Both are assumed to be generated by the same logging policy, so $p_j=\pi_0(a_j|x_j)$ for both sets.} 

We will next develop new theory (Sections~\ref{Sec: var bounds} and \ref{Sec: Semi-supervised CRM Algorithms}) and propose two novel algorithms (Section~\ref{sec:experiments}) to learn a policy that minimizes the true risk using logged missing-feedback and known-feedback datasets.

\section{Bounds on True Risk of IPS Estimator}\label{Sec: var bounds}
In this section we provide an upper bound on the variance of importance weighted cost, i.e., 
\begin{equation}
    \begin{split}
    & \operatorname{Var}\left(w(A,X)c(A,X)\right)
    \\&\triangleq \mathbb{E}_{P_X\otimes\pi_{0}(A|X)}\left[\left(w(A,X)c(A,X)\right)^2\right]- R(\pi_\theta)^2,
 \end{split}
\end{equation}
where $R(\pi_\theta)=\mathbb{E}_{P_X\otimes\pi_{0}(A|X)}\left[w(A,X)c(A,X)\right]=\mathbb{E}_{P_X\otimes\pi_{\theta}(A|X)}\left[c(A,X)\right]$. 

Throughout this section we use the simplified notations $\KLr(\pi_\theta\|\pi_0)=\KLr(\pi_\theta(A|X)\|\pi_0(A|X))$ and
$\KLr(\pi_0\|\pi_\theta)=\KLr(\pi_0(A|X)\|\pi_\theta(A|X))$.
All the proofs are deferred to the App.\ref{App: Proofs var}.
\begin{proposition}\label{Prop: bound KL Var} Suppose that the importance weighted of squared cost function, i.e., $w(A,X)c^2(A,X)$, is $\sigma$-sub-Gaussian\footnote{A random variable $X$ is $\sigma$-subgaussian if $E[e^{\gamma(X-E[X])}]\leq e^{\frac{\gamma^2 \sigma^2}{2}}$ for all $\gamma \in \mathbb{R}$.}  under $P_X\otimes \pi_0(A|X)$ and $P_X\otimes \pi_\theta(A|X)$, and the cost function has bounded range $[b_1,b_2]$ with $b_2\geq 0$. Then, the following upper bound holds on the variance of the importance weighted cost function:
\begin{equation}
    \begin{split}
    & \operatorname{Var}\left(w(A,X)c(A,X)\right)\\&\leq 
    \sqrt{2\sigma^2 \min(\KLr(\pi_\theta\|\pi_0),\KLr(\pi_0\|\pi_\theta)) }+b_u^2-b_l^2,
 \end{split}
\end{equation}
where $b_l=\max(b_1,0)$ and $b_u=\max(|b_1|,b_2)$. 
\end{proposition}
We explore the connection between sub-Gaussian assumption and uniform coverage assumption \cite{wang2023oracle,gabbianelli2023importance} in App.\ref{App: Proofs var}. We have the following Corollary for the truncated IPS estimator with threshold $\nu\in (0,1] $.

\begin{corollary}\label{cor: bound KL Var}
Assume a bounded cost function with range $[b_1,0]$ and a truncated IPS estimator with threshold $\nu\in(0,1]$. Then the following upper bound holds on the variance of the truncated importance weighted cost function,
\begin{equation}
    \begin{split}
    &\operatorname{Var}\left(w_\nu(A,X)c(A,X)\right) 
    \\&\leq b_1^2(\nu^{-1}\sqrt{ \min(\KLr(\pi_\theta\|\pi_0),\KLr(\pi_0\|\pi_\theta))/2 }+1).
        \end{split}
\end{equation}
\end{corollary}
Using \citet[Lemma~1]{cortes2010learning}, we can provide an upper bound on the variance of importance weights in terms of the chi-square divergence by considering $c(a,x)\in[b_1,b_2]$, as follows:
\begin{equation}\label{Eq: Chi-square div}
    \begin{split}    
    &\operatorname{Var}\left(w(A,X)c(A,X)\right)
    \\&\leq b_u^2 \chi^2(\pi_\theta(A|X)\|\pi_0(A|X))+b_u^2-b_l^2,
\end{split}
\end{equation}
where $b_l=\max(b_1,0)$ and $b_u=\max(|b_1|,b_2)$. In App.\ref{App: comp kl vs chi}, we discuss that, under some conditions, the upper bound in Proposition~\ref{Prop: bound KL Var} is tighter than the upper bound based on chi-square divergence in~\eqref{Eq: Chi-square div}. 
The upper bound in Proposition~\ref{Prop: bound KL Var} shows that we can reduce the variance of importance weighted cost function, i.e., $w(A, X)c(A,X)$, by minimizing the KL divergence or reverse KL divergence, i.e. $\KLr(\pi_\theta\|\pi_0)$ or $\KLr(\pi_0\|\pi_\theta)$. A lower bound on the variance of the importance weighted cost function in terms of the KL divergence $\KLr(\pi_\theta\|\pi_0)$ is provided in App.\ref{App: Proofs var}.

We can derive a high-probability bound on the true risk under the truncated IPS estimator using the upper bound on the variance of importance weighted cost function in Corollary~\ref{cor: bound KL Var}.
\begin{theorem}\label{Theorem: main result}
Suppose the cost function takes values in $[-1,0]$. Then, for any $\delta\in(0,1)$, the following bound on the true risk of policy $\pi_\theta(A|X)$ with the truncated IPS estimator (with parameter $\nu\in(0,1]$) holds with probability at least $1-\delta$ under the distribution $P_X \otimes \pi_0(A|X)$:
\begin{equation}
 \begin{split}
        R(\pi_\theta)&\leq  \hat{R}_\nu(\pi_\theta,S)+ \frac{2 \log(\frac{1}{\delta})}{3\nu n}\\&\quad+\sqrt{\frac{ (\nu^{-1}\sqrt{ 2M }+2)   \log(\frac{1}{\delta})}{n}},
 \end{split}
\end{equation}
where $M=\min\{ \KLr(\pi_\theta\|\pi_0),\KLr(\pi_0\|\pi_\theta) \}$.
\end{theorem}
The proof of Theorem~\ref{Theorem: main result} leverages the Bernstein inequality together with an upper bound on the variance of importance weighted cost function using Proposition~\ref{Prop: bound KL Var}. Theorem~\ref{Theorem: main result} shows that we can minimize the KL divergence  $\KLr(\pi_\theta(A|X)\|\pi_0(A|X))$, or reverse KL divergence $\KLr(\pi_0(A|X)\|\pi_\theta(A|X))$, instead of the empirical variance minimization in CRM framework \citep{swaminathan2015batch} which is inspired by the upper bound given by~\citet{maurer2009empirical}. We compared our upper bound with that of \citet[Theorem~1]{london_sandler2019} in App.\ref{app: compare theory bcrm}.

Note that as we assumed the truncated IPS estimator, we do not need the overlap assumption \footnote{{\blue Given $\pi_{\theta}(A|X)$ and $\pi_0(A|X)$, then the overlap assumption between learning policy and logging policy holds if there exists $B>0$ such that $\sup_{\theta\in\Theta}\sup_{(a,x)\in\mathcal{X}\times \mathcal{A}}\frac{\pi_{\theta}(A=a|X=x)}{\pi_0(A=a|X=x)}\leq B$.}} \citep{mandal2023performative} as in off-policy reinforcement learning. 

The minimization of KL divergence and reverse KL divergence can also be interpreted from another perspective.
\begin{proposition}\label{Prop: bound KL}
The following upper bound holds on the absolute difference between risks of logging policy $\pi_0(a|x)$ and the policy $\pi_\theta(a|x)$:
\begin{equation*}
     |R(\pi_\theta)-R(\pi_0)|\leq \min\left(\sqrt{\frac{\KLr(\pi_\theta\|\pi_0)}{2}},\sqrt{\frac{\KLr(\pi_0\|\pi_\theta)}{2}}\right).
\end{equation*}
\end{proposition}
Based on Proposition~\ref{Prop: bound KL}, minimizing KL divergence and reverse KL divergence would lead to a policy close to the logging policy in KL divergence or reverse KL divergence. This phenomenon, which is also observed in the works by \citet{swaminathan2015batch,wu2018variance,london_sandler2019}, is aligned with the fact that the learned policy should not diverge too much from the logging policy~\citep{schulman2015trust}. As mentioned by \citet{brandfonbrener2021offline} and \citet{swaminathan2015self}, the propensity overfitting issues are solved by variance reduction. Therefore, with the KL divergence and reverse KL divergence regularization, we can reduce the propensity overfitting.

\section{Semi-supervised Batch Learning via Feedback Free Regularization}\label{Sec: Semi-supervised CRM Algorithms}
We now propose our approach for S2BL settings: feedback-free regularization. It can leverage the availability of the logged known-feedback dataset $S$ and the logged missing-feedback dataset $S_{u}$. The feedback-free regularized semi-supervised batch learning is based on optimizing a regularized batch learning objective via logged data, where the regularization function is independent of the feedback. It is inspired by an entropy minimization approach in semi-supervised learning, where one optimizes a label-free entropy function.

Note that the KL divergence $\KLr(\pi_\theta\|\pi_0)$ and reverse KL divergence $\KLr(\pi_0\|\pi_\theta)$  appearing in Theorem~\ref{Theorem: main result} are independent of the cost function values (feedback). This motivates us to consider them as functions that can be optimized using both the logged known-feedback and missing-feedback datasets. It is worth mentioning that the regularization based on empirical variance proposed by \citet{swaminathan2015batch} depends on feedback. 

We propose the following truncated IPS estimator regularized by KL divergence $\KLr(\pi_\theta\|\pi_0)$ or reverse KL divergence $\KLr(\pi_0\|\pi_\theta)$, thus casting S2BL into a semi-supervised CRM problem for $\lambda>0$,
\begin{align*}
    &\hat{R}_{\mathrm{KL}}(\pi_\theta,S,S_{u})\triangleq    \hat{R}_\nu(\pi_\theta,S)+ \lambda \KLr(\pi_\theta(A|X)\|\pi_0(A|X)),
    \\
    &\hat{R}_{\mathrm{RKL}}(\pi_\theta,S,S_{u}) \triangleq  \hat{R}_\nu(\pi_\theta,S)+ \lambda \KLr(\pi_0(A|X)\|\pi_\theta(A|X)), 
\end{align*}
where for $\lambda=0$, our problem reduces to traditional batch learning with the logged known-feedback dataset that neglects the logged missing-feedback dataset. Note that, various works have suggested the use of KL regularization, and we conducted a comparative analysis between our work and these studies in App.\ref{App: comp KL}. 

We provide a regret upper bound \footnote{The regret is defined as $|R(\pi^\star_\theta)- R(\pi^r_{\theta})|$, where the solution to our KL-regularized risk minimization is denoted by $\pi^r_{\theta}$.} of our algorithms (KL-regularized risk minimization) in App.\ref{app: regret upper bound}. In addition, we study the optimal policy under KL regularization in App.\ref{App: true risk minimizer}. 

For the estimation of $\KLr(\pi_\theta(A|X)\|\pi_0(A|X))$ and $\KLr(\pi_0(A|X)\|\pi_\theta(A|X))$, we can apply the logged missing-feedback dataset as follows:
\begin{align}
   &\hat{L}_{\mathrm{KL}}(\pi_\theta)\triangleq\sum_{i=1}^k \frac{1}{m_{a_i}}\sum_{\substack{(x,a_i,p)\in\\ S_{u}\cup S}} A_{\mathrm{KL}}(x,a_i,p)
   \label{Eq: CRM regularize by KL}
    \\
    &\hat{L}_{\mathrm{RKL}}(\pi_\theta)\triangleq\sum_{i=1}^k \frac{1}{m_{a_i}}\sum_{\substack{(x,a_i,p)\in\\ S_{u}\cup S}} A_{\mathrm{RKL}}(x,a_i,p)
    \label{Eq: CRM regularize estimator}
\end{align}
where $A_{\mathrm{KL}}(x,a_i,p)= \pi_\theta(a_i|x)\log(\pi_\theta(a_i|x))-\pi_\theta(a_i|x)\log(p)$, $A_{\mathrm{RKL}}(x,a_i,p)=-p\log(\pi_\theta(a_i|x))+ p\log(p)$ and $m_{a_i}$ is the number of context, action, and propensity score tuples, i.e., $(x,a,p)\in S_{u}\cup S$, with the same action, e.g., $a=a_i$ (note we have $\sum_{i=1}^k m_{a_i} =m+n$). 
It is possible to show that these estimators of KL divergence and reverse KL divergence are unbiased in the asymptotic sense.
\begin{proposition}\label{Prop: estimators}(proved in App.\ref{App: Proofs semi}) Suppose that $\KLr(\pi_\theta(A|X)\|\pi_0(A|X))$ and the reverse \newline$\KLr(\pi_0(A|X)\|\pi_\theta(A|X))$ are bounded.
Assuming $m_{a_i}\rightarrow \infty$ $(\forall a_i\in \mathcal{A})$, then
 $\hat{L}_{\mathrm{KL}}(\pi_\theta)$ and $\hat{L}_{\mathrm{RKL}}(\pi_\theta)$ are unbiased estimations of $\KLr(\pi_\theta(A|X)\|\pi_0(A|X))$ and $\KLr(\pi_0(A|X)\|\pi_\theta(A|X))$, respectively.
\end{proposition}
An estimation error analysis for the proposed estimators in Proposition~\ref{Prop: estimators} is conducted in App.\ref{App: Proofs semi}.
Note that another approach to minimize the KL divergence or reverse KL divergence is $f$-GAN \citep{wu2018variance,nowozin2016}, which is based on using a logged known-feedback dataset without considering feedback and propensity scores. It is worthwhile to mention that the generative-adversarial approach will not consider propensity scores in the logged known-feedback dataset and also incur more complexity, including Gumbel softmax sampling~\citep{jang2016categorical} and discriminator network optimization. We proposed a new estimator of these information measures considering our access to propensity scores in the logged missing-feedback dataset.
Since the term $p\log(p)$ in \eqref{Eq: CRM regularize estimator} is independent of policy $\pi_\theta$, we ignore it and optimize the following quantity instead of $\hat{L}_{\mathrm{RKL}}(\pi_\theta,S_{u})$ which is similar to cross-entropy by considering propensity scores as weights of cross-entropy:
\begin{equation}\label{eq: rkl final estim}
    \hat{L}_{\mathrm{WCE}}(\pi_\theta)\triangleq\sum_{i=1}^k \frac{1}{m_{a_i}}\sum_{\substack{(x,a_i,p)\\\in S_{u}\cup S}} -p\log(\pi_\theta(a_i|x)).
\end{equation}

\section{Algorithms and Experiments}\label{sec:experiments}
We briefly present our experiments. More details and discussions can be found in App.\ref{App: experiments}. 
We consider two approaches: softmax policy with linear model inspired by \citet{swaminathan2015batch,london_sandler2019}, and the softmax policy via deep model inspired by \citet{joachims2018deep}.

\textbf{Softmax policy with linear model:} Following the prior works of \citet{swaminathan2015batch,london_sandler2019}, we consider the stochastic softmax policy 
\begin{equation}
    \begin{split}
        \pi_{\tilde{\theta}}(a_i|x)=\frac{\exp(\tilde{\theta}.\phi(a_i,x))}{\sum_{j=1}^k \exp(\tilde{\theta}.\phi(a_j,x))},
    \end{split}
\end{equation}
where $\phi(a_i,x)$ is a feature map for $(a_i,x)$ and $\tilde{\theta}$ is the vector of parameters. Therefore, our learning policy is based on a linear model. 

\textbf{Softmax policy with deep model:} Following \citet{joachims2018deep}, we consider the output of a softmax layer in a neural network as a stochastic learning policy,
\begin{equation}
    \pi_{\theta}(a_i|x)=\frac{\exp(h_\theta(x,a_i))}{\sum_{j=1}^k \exp(h_\theta(x,a_j))},
\end{equation}
 where $h_\theta(x,a_i)$ is the $i$-th input to softmax layer for context $x\in \mathcal{X}$ and action $a_i\in \mathcal{A}$.

 \textbf{Baselines:} For linear model, we consider the Bayesian CRM, cf. \citet{london_sandler2019}, as a baseline to compare with our algorithms. More details for comparison of our algorithm with Bayesian CRM in provided in App.\ref{app: compare alg bcrm}. For deep model, we consider the BanditNet as a baseline in our experiment. More details regarding the BanditNet is provided in App.\ref{App: baselines}.

 \textbf{Algorithms:} The WCE-S2BL algorithm, proposed in Algorithm~\ref{Alg: WCE-S2BL}, is based on feedback-free regularized truncated IPS estimator in linear model via truncated weighted cross-entropy. 
The KL-S2BL algorithm is similar to Algorithm~\ref{Alg: WCE-S2BL} by replacing $\hat{L}_{\mathrm{WCE}}(\theta^{t_g})$ with $\hat{L}_{\mathrm{KL}}^\nu(\theta^{t_g})$ defined as
\begin{align*}
    \hat{L}_{\mathrm{KL}}^\nu(\theta^{t_g})=\sum_{i=1}^k \frac{1}{m_{a_i}}\sum_{\substack{(x,a_i,p) \\ \in S_{u}\cup S}} \pi_{\theta^{t_g}}(a_i|x)\log \left(\frac{\pi_{\theta^{t_g}}
   (a_i|x)}{\max(\nu,p)}\right).
\end{align*}
We examine the performance of the algorithms WCE-S2BL and KL-S2BL in both linear and deep models. For a fair comparison, we run experiments for WCE-S2BL and KL-S2BL using the logged known-feedback dataset for regularization. These algorithms are referred to as WCE-S2BLK and KL-S2BLK, respectively. Note that in linear model, we have truncated IPS estimator. However, in the deep model, we consider BanditNet which is based on self-normalized IPS estimator. Therefore, in the WCE-S2BL algorithm for deep model we replace the truncated IPS estimator via BanditNet approach \cite{joachims2018deep} in Algorithm~\ref{Alg: WCE-S2BL}.
\begin{algorithm}[ht!]
\caption{WCE-S2BL Algorithm for Linear Model}
\label{Alg: WCE-S2BL}
\begin{algorithmic}[1]
 \INPUT{ $S=(x_i,a_i,p_i,c_i)_{i=1}^n$ sampled from $\pi_0$, $S_{u}=(x_j,a_j,p_j)_{j=1}^m$ sampled from $\pi_0$, hyper-parameters $\lambda$ and $\nu$, initial policy $\pi_{\theta^0}(a|x)$, epoch index $t_g$ and max epochs for the whole algorithm $M$}
  
 \OUTPUT{ An optimized policy $\pi^\star_\theta(a|x)$ which minimize the regularized risk by weighted cross-entropy}

   \STATE For $t_g$  $\leq M$, sample $n$ samples $(x_i,a_i,p_i,c_i)$ from $S$ and estimate the re-weighted loss as $\hat{R}_\nu(\theta^{t_g})=\frac{1}{n}\sum_{i=1}^{n} c_i \frac{\pi_{\theta^{t_g}}(a_i|x_i)}{\max(\nu,p_i)}$. 
   	
   	\STATE Get the gradient with respect to $\theta^{t_g}$ as $g_1\leftarrow \nabla_{\theta^{t_g}}\hat{R}_\nu(\theta^{t_g}).$
    Sample $m$ samples from $ S_{u}$ and estimate the weighted cross-entropy loss ($\sum_{i=1}^k m_{a_i} = m$).

    \STATE Compute $\hat{L}_{\mathrm{WCE}}(\theta^{t_g})$\\ $=\sum_{i=1}^k \frac{1}{m_{a_i}}\sum_{(x,a_i,p)\in S_{u}\cup S} -p\log(\pi_{\theta^{t_g}}(a_i|x))$.
   
   	\STATE Get the gradient with respect to $\theta^{t_g}$ as\\
   	 $g_2\leftarrow \nabla_{\theta^{t_g}} \hat{L}_{\mathrm{WCE}}(\theta^{t_g})$.
   	 
   	\STATE Update $\theta^{t_g+1}=\theta^{t_g}-( g_1+ \lambda g_2)$.
   	
   	\STATE $t_g=t_g+1.$
   
   \end{algorithmic}
\end{algorithm}

\textbf{Datasets:} We apply the standard supervised to bandit transformation~\citep{beygelzimer2009offset} on two image classification datasets: Fashion-MNIST (FMNIST)~\citep{xiao2017/online} and CIFAR-10~\citep{krizhevsky2009learning}.
This transformation assumes that each of the ten classes in the datasets corresponds to an action. Then, a logging policy stochastically selects an action for every sample in the dataset. For each data sample $x$, action $a$ is sampled by logging policy. For the selected action, propensity score $p$ is determined by the softmax value of that action. If the selected action matches the actual label assigned to the sample, then we have $c=-1$, and $c=0$ otherwise. So, the 4-tuple $(x, a, p, c)$ makes up the dataset.{\blue In App.\ref{App: experiments}, we also consider other datasets, including CIFAR-100 and EMNIST and the real dataset Kuairec.}

\textbf{Logging policy:} 
To create logging policies with different performances, given inverse temperature \footnote{The inverse temperature $\tau$ is defined as $\pi_{0}(a_i|x)=\frac{\exp(h(x,a_i)/\tau)}{\sum_{j=1}^k \exp(h(x,a_j)/\tau)}$ where $h(x,a_i)$ is the $i$-th input to the softmax layer for context $x\in\mathcal{X}$ and action $a_i\in\mathcal{A}$. } $\tau\in\{1, 5, 10, 20\}$ we train a simplified ResNet architecture having a single residual layer in each block with inverse temperature $\tau$ in the softmax layer on the fully-labeled dataset, FMNIST. For CIFAR-10, we use linear model for logging policy, using pre-trained features as image representation. Then, we augment the dataset with the outputs and feedback of the trained policy, this time with inverse temperature equal to $1$ in the softmax layer. Hence, the learned policy is logged with inverse temperature $\tau$. Increasing $\tau$ leads to more uniform and less accurate logging policies. 

We evaluate the performance of the different algorithms based on the accuracy of the trained model. Inspired by \citet{london_sandler2019}, we calculate the accuracy for a deterministic policy where the accuracy of the model based on the argmax of the softmax layer output for a given context is computed.

To simulate the absence of feedback for logged missing-feedback datasets, we pretended that the feedback (cost) was only available in $\rho\in\{0.02,  0.2\}$ of the samples in each dataset, while the feedback of the remaining samples is missed. Recall that the regularization term is minimized via both logged known-feedback and logged missing-feedback datasets. 

For each value of $\tau$ and $\rho$ and for both types of deep and linear models, we apply WCE-S2BL, KL-S2BL, WCE-S2BLK and KL-S2BLK, and observe the accuracy over three runs.
Figure~\ref{fig:MeanAccuracyDeep} shows the accuracy of WCE-S2BL, WCE-S2BLK, KL-S2BL and KL-S2BLK  methods compared to BanditNet \cite{joachims2018deep} for the deep model approach, for $\tau=10$ and different number of known-feedback samples, in the FMNIST and CIFAR-10 datasets. The error bars represent the standard deviation over the three runs.
Figure~\ref{fig:MeanAccuracyLinear} shows similar results for the linear model.
Table~\ref{tab:comparison main-banditnet} shows the deterministic accuracy of WCE-S2BL, KL-S2BL, WCE-S2BLK, KL-S2BLK and BanditNet methods for $\tau\in \{1,10\}$, and $\rho\in\{0.02, 0.2\}$.
More results for other values of $\tau$ and $\rho$ are available in App.\ref{app: more results}. More experiments about the effect of logged missing-feedback dataset and the minimization of regularization terms are available at App.\ref{app: effect miss}.

\begin{figure}[htb!]
\centering
  \begin{subfigure}[FashionMNIST]
      {\includegraphics[width=0.45\columnwidth]{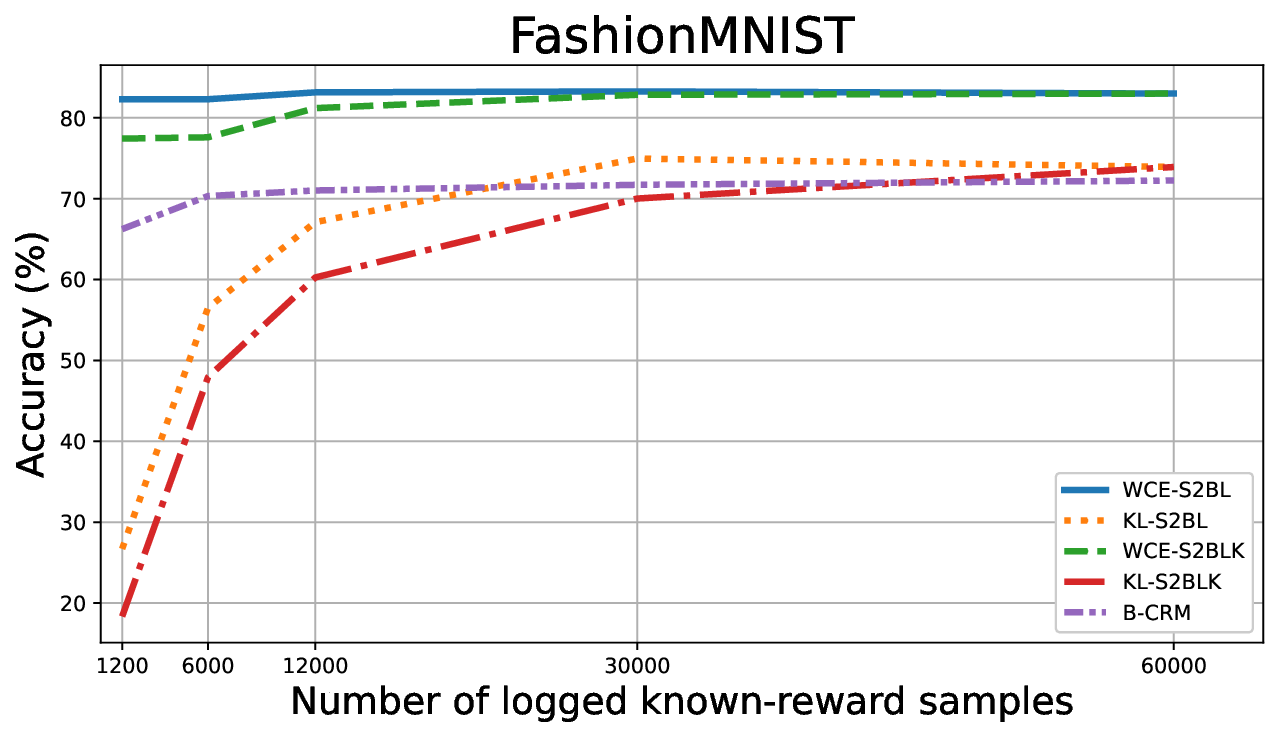}}
  \end{subfigure}
  \hfill
  \begin{subfigure}[CIFAR-10]{
     \includegraphics[width=0.45\columnwidth]{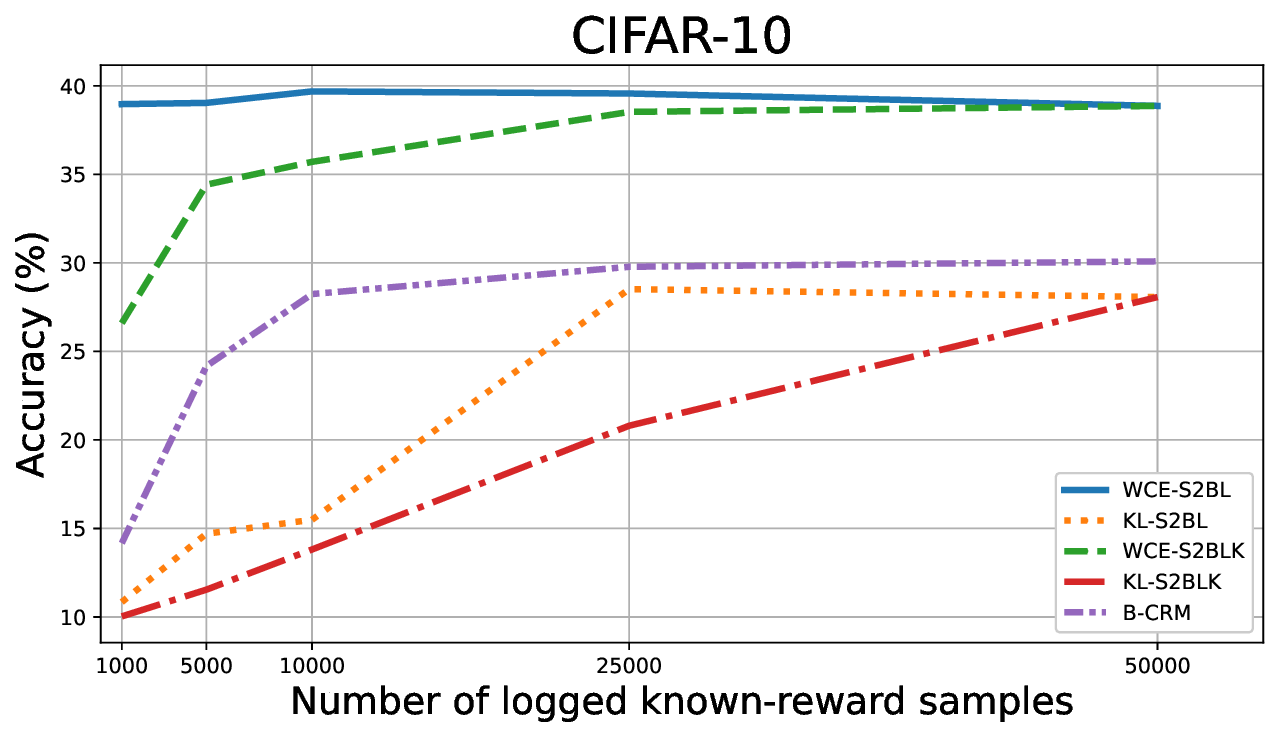}}
  \end{subfigure}
  \caption{Accuracy of WCE-S2BL, KL-S2BL,WCE-S2BLK, KL-S2BLK, and B-CRM for $\tau=10$.}
  \label{fig:MeanAccuracyLinear}
\end{figure}
\begin{figure}[htb!]
\centering
  \begin{subfigure}[FashionMNIST]
   { \includegraphics[width=0.45\columnwidth]{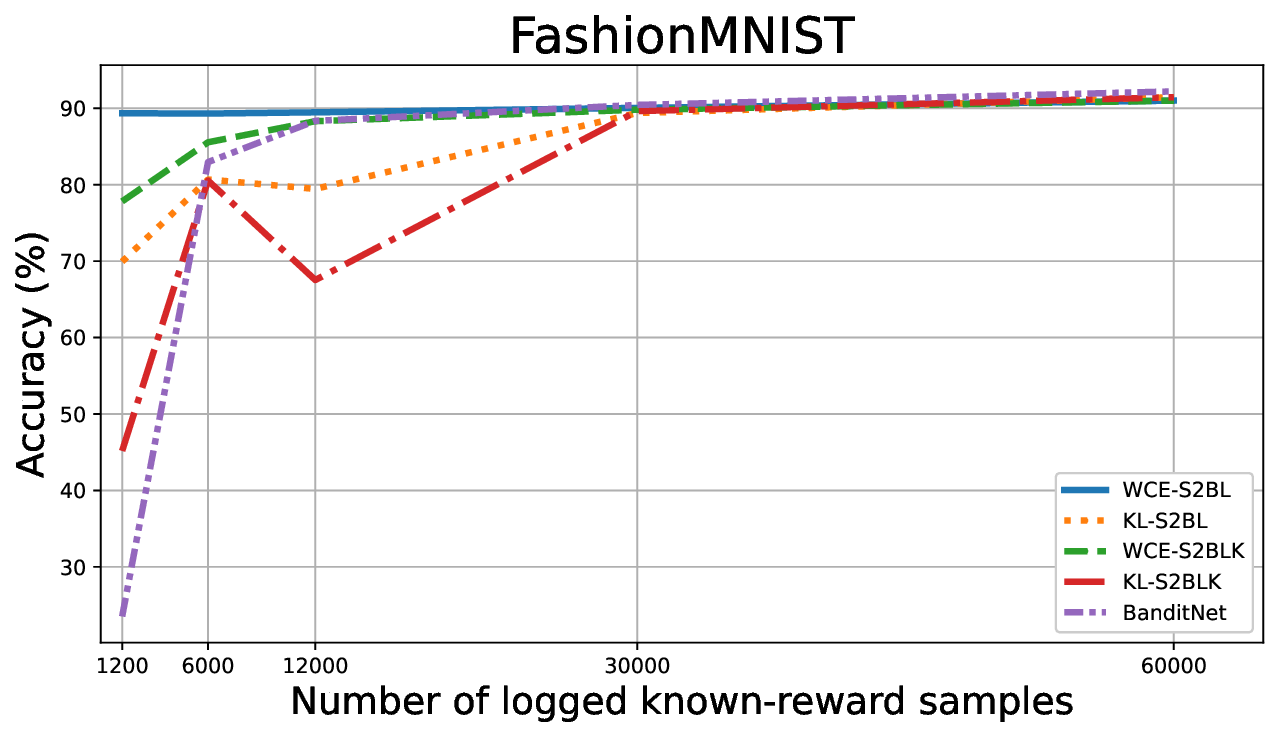}}
  \end{subfigure}
  \begin{subfigure}[CIFAR-10]
   { \includegraphics[width=0.45\columnwidth]{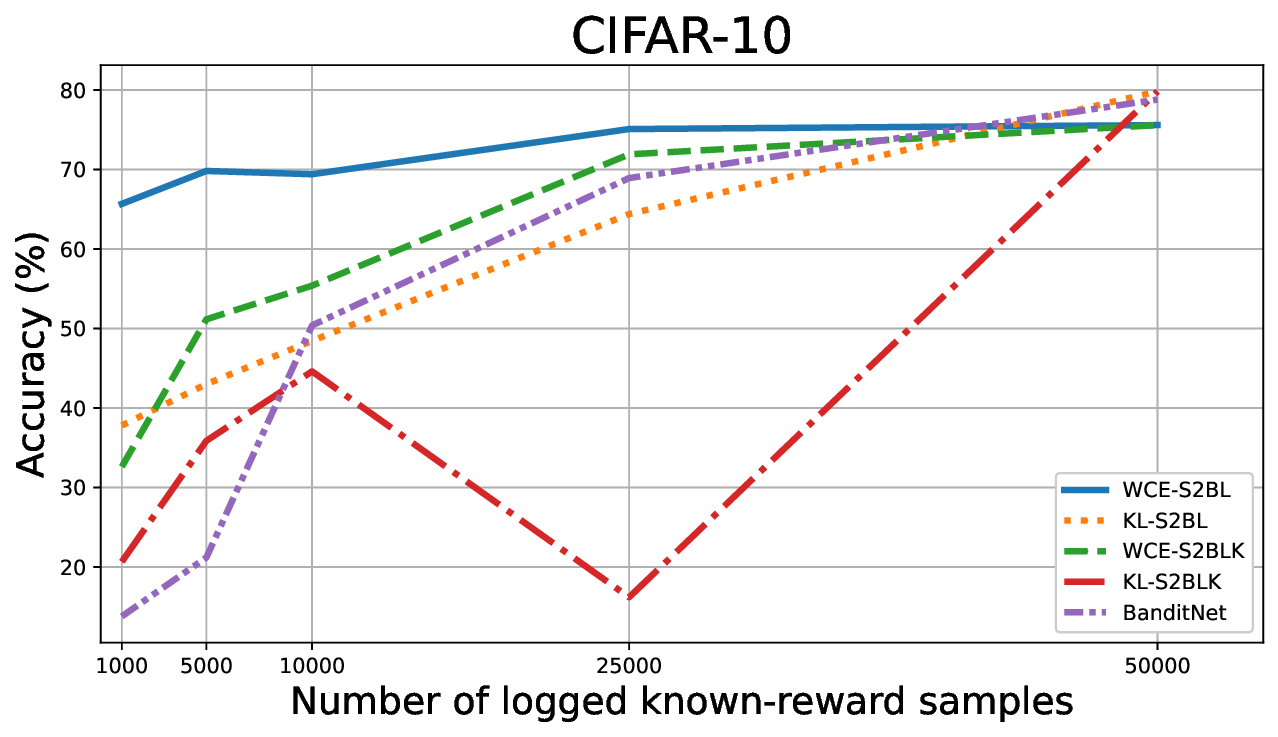}}
  \end{subfigure}
  \caption{Accuracy of WCE-S2BL, KL-S2BL,WCE-S2BLK, KL-S2BLK, and BanditNet for $\tau=10$.}
  \label{fig:MeanAccuracyDeep}
\end{figure}

\begin{table*}[htb]
  \caption{Comparison of different algorithms WCE-S2BL, KL-S2BL, WCE-S2BLK, KL-S2BLK and Bayesian-CRM (B-CRM) deterministic accuracy for FMNIST and CIFAR-10 with linear layer setup and for different qualities of logging policy ($\tau\in\{1,10\}$) and proportions of labeled data ($\rho\in\{0.02,0.2\}$). }
  \label{tab:comparison linear}
  \centering
  \resizebox{\textwidth}{!}{
	\begin{tabular}{ccccccccc}
	\\
 \toprule
    	Dataset & $\tau$ & $\rho$ & \textbf{WCE-S2BL} & \textbf{KL-S2BL} & \textbf{WCE-S2BLK} & \textbf{KL-S2BLK} & \textbf{B-CRM} & \textbf{Logging Policy}\\
    	\midrule
    	\multirow{4}{*}{FMNIST} & \multirow{2}{*}{1} & 0.02 &$84.37\pm 0.14$ & $71.67\pm 0.26$ & $78.84\pm 0.05$ & $74.71\pm 0.06$ & $64.67\pm 1.44$
 & \multirow{2}{*}{\bm{$91.73$}}\\
    	 &  & 0.2 & $83.59\pm 0.18$ & $71.88\pm 0.31$ & $83.05\pm 0.06$ & $74.06\pm 0.00$ & $70.99\pm 0.32$ &\\
      \cmidrule(lr{0.5em}){2-9}
    	 & \multirow{2}{*}{10} & 0.02 & \bm{$82.31\pm 0.07$} & $26.71\pm 2.18$ & $77.43\pm 0.13$ & $18.35\pm 7.06$ & $66.24\pm 00.03$
 & \multirow{2}{*}{$20.72$} \\
    	 & & 0.2 & \bm{$83.15\pm 0.09$} & $67.10\pm 5.17$ & $81.20\pm 0.12$ & $60.26\pm 0.88$ & $71.02\pm 0.30$ & \\
      \midrule
     \multirow{4}{*}{CIFAR-10} &  \multirow{2}{*}{1} & 0.02 & \bm{$62.95\pm0.08$} & $28.29\pm11.35$ & $9.49\pm0.72$ & $10.02\pm0.02$ & $55.02\pm0.14$ 
 & \multirow{2}{*}{$52.89$}\\
    	 &  & 0.2 & \bm{$62.83\pm0.06$} & $26.29\pm6.92$ & $62.90\pm0.10$ & $14.08\pm1.58$ & $57.75\pm0.42$ & \\
      \cmidrule(lr{0.5em}){2-9}
    	 & \multirow{2}{*}{10} & 0.02 & \bm{$54.47\pm1.34$} & $11.60\pm1.13$ & $41.93\pm1.25$ & $10.08\pm0.12$ & $44.66\pm0.29$
 & \multirow{2}{*}{$36.6$} \\
    	 & & 0.2 & \bm{$56.99\pm0.00$} & $22.83\pm0.46$ & $56.94\pm0.19$ & $13.69\pm2.69$ & $52.09\pm0.43$ &  \\
      \bottomrule
  \end{tabular}}
\end{table*}

\begin{table*}[htb]
  \caption{Comparison of different algorithms WCE-S2BL, KL-S2BL, WCE-S2BLK, KL-S2BLK and BanditNet deterministic accuracy for FMNIST and CIFAR-10 with deep model setup and different qualities of logging policy ($\tau\in\{1,10\}$) for different proportions of labeled data ($\rho\in\{0.02,0.2\}$). }
  \label{tab:comparison main-banditnet}
  \centering
  \resizebox{\textwidth}{!}{
	\begin{tabular}{ccccccccc}
	\\
 \toprule
    	Dataset & $\tau$ & $\rho$ & \textbf{WCE-S2BL} & \textbf{KL-S2BL} & \textbf{WCE-S2BLK} & \textbf{KL-S2BLK} & \textbf{BanditNet} & \textbf{Logging Policy}\\
    	\midrule
    	\multirow{4}{*}{FMNIST} & \multirow{2}{*}{1} & 0.2 & $\bm{93.16 \pm 0.18}$ & $92.04 \pm 0.13$  & $82.76\pm 4.45$ & $87.72\pm 0.53$ & $89.60\pm 0.49$ & $91.73$\\
    	 &  & 0.02 & $\bm{93.12\pm 0.16}$ & $91.79\pm 0.16$  & $78.66\pm 0.90$ & $61.46\pm 9.97$ & $78.64\pm 1.97$ & $91.73$\\
      \cmidrule(lr{0.5em}){2-9}
    	 & \multirow{2}{*}{10} & 0.2 & $\bm{89.47 \pm 0.3}$ & $79.45 \pm 0.75$  & $88.31\pm 0.14$ & $67.53\pm 2.06$ & $88.35\pm 0.45$ & $20.72$\\
    	 & & 0.02 & $\bm{89.35 \pm 0.15}$ & $69.94 \pm 0.60$  & $77.82\pm 0.73$ & $45.18\pm 19.82$ & $23.52\pm 3.15$ & $20.72$\\
      \midrule
     \multirow{4}{*}{CIFAR-10} &  \multirow{2}{*}{1} & 0.2 & $85.06 \pm 0.32$ & \bm{$85.53 \pm 0.56$}  & $58.04\pm 5.47$ & $54.12\pm 0.51$ &  $67.96\pm 0.62$ & $79.77$\\
     
    	 &  & 0.02 & \bm{$85.01 \pm 0.37$} & $84.60 \pm 0.65$  & $17.12\pm 0.97$ & $21.63\pm 1.44$ &  $27.39\pm 3.47$ & $79.77$ \\
      \cmidrule(lr{0.5em}){2-9}
    	 & \multirow{2}{*}{10} & 0.2 & $\bm{69.40 \pm 0.47}$ & $48.44 \pm 0.26$  & $55.38\pm 3.63$ & $44.60\pm 0.19$ & $50.38 \pm 0.55$ & $43.45$ \\
    	 & & 0.02 & $\bm{65.67 \pm 1.06}$ & $37.80 \pm 0.85$  & $32.61\pm 1.14$ & $20.66\pm 5.74$ & $13.78\pm 1.99$& $43.45$  \\
      \bottomrule
  \end{tabular}}
\end{table*}
Our methods achieve maximum accuracy even when the logging policy's accuracy is not well. For example, in Tables~\ref{tab:comparison main-banditnet} for the CIFAR-10 in deep model setup with $\tau=10$ and $\rho=0.2$, we observe \bm{$69.40\pm 0.47$} accuracy for WCE-S2BL in comparison with \bm{$50.38\pm 0.55$} and \bm{$43.45$} for BanditNet and logging policy, respectively. 
\subsection{Effect of logged missing-feedback dataset and minimization of the regularization}\label{app: effect miss}

We also run experiments to investigate the effect of the size of logged missing-feedback dataset. For this purpose, we fix the number of logged known-feedback dataset to $1000$ samples for CIFAR-10 and $1200$ for FMNIST. Then, we add $1000$, $4000$, $9000$, $24000$ missing-feedback samples to the dataset and compute the accuracy of the learning policy. Figure~\ref{fig:logged_missing_reward} shows the accuracy for different numbers of added missing-feedback samples for CIFAR-10 and FMNIST datasets over different ratio of logged missing-feedback samples to logged known-feedback samples. 
We observe that by increasing the number of missing-feedback logged data samples (the ratio of logged missing-feedback samples to logged known-feedback samples with fixed logged known-feedback sample size), the deterministic accuracy is improved. To provide more insight with respect to minimization of regularization, we run some experiments for deep model, to investigate the performance if we just minimize the regularization terms, i.e., KL divergence or reverse KL divergence, via the logged known-feedback dataset and missing-feedback datasets. The results are shown in Table~\ref{tab:comparison onle reg}. It can be noted that, under all circumstances, it is essential to minimize the regularized version of BanditNet for better accuracy. Therefore, both main loss and regularization are needed for better performance.
\begin{figure}[htb!]
\centering
  \begin{subfigure}[FashionMNIST]
    {\includegraphics[width=0.45\columnwidth]{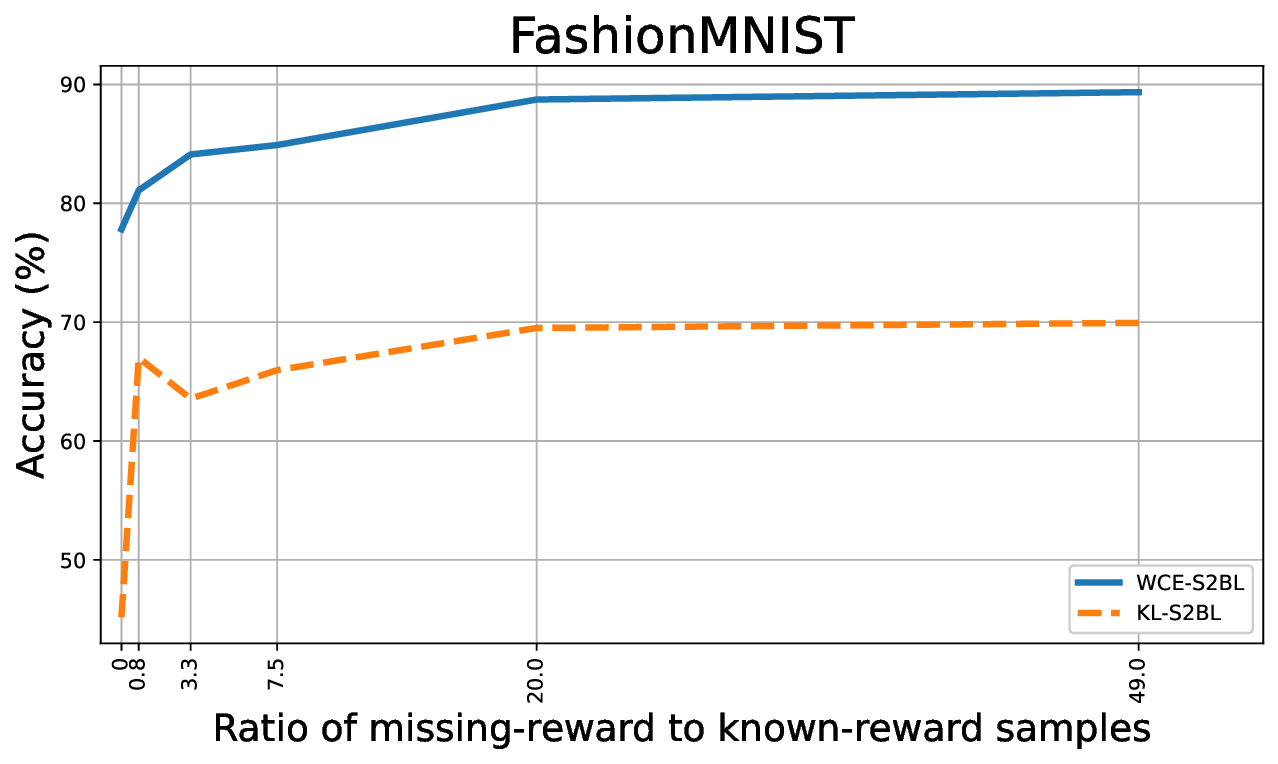}}
  \end{subfigure}
  \begin{subfigure}[CIFAR-10]
    {\includegraphics[width=0.45\columnwidth]{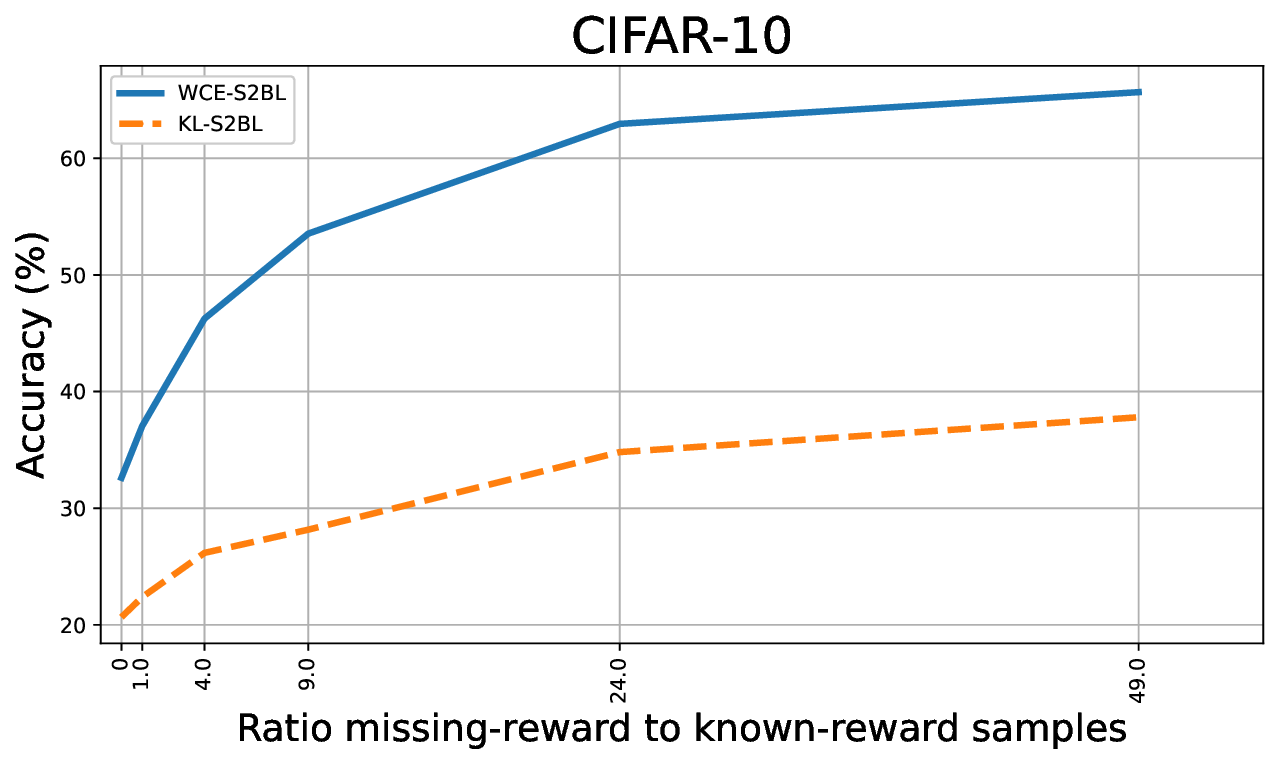}}
  \end{subfigure}
  \caption{Accuracy of WCE-S2BL and KL-S2BL for different ratio of missing-feedback data samples to known-feedback data samples. We fix the number of known-feedback data samples to $1000$ samples.}
  \label{fig:logged_missing_reward}
\end{figure}

\begin{table}[htb]
  \caption{Comparison of WCE-S2BL, KL-S2BL deterministic accuracy trained with $\rho=0.1$ and their counterpart without Self- normalized IPS (SNIPS)  as main loss in BanditNet. Accuracy on FMNIST and CIFAR-10 datasets is reported for $\tau\in\{1, 5, 10, 20\}$.}
  \label{tab:comparison onle reg}
  \centering
  \resizebox{\linewidth}{!}{
	\begin{tabular}{ccccccccc}
	\\
 \toprule
    	Dataset & $\tau$ & \textbf{WCE-S2BL} & \textbf{WCE-S2BL w/o SNIPS} & \textbf{KL-S2BL} & \textbf{KL-S2BL w/o SNIPS} & \\
    	\midrule
    	\multirow{4}{*}{FMNIST} & 1 & $\bm{93.26\pm 0.05}$ & $92.89\pm 0.09$ & $91.73\pm 0.08$ & $85.51\pm 1.20$\\
    	 & 5 & $\bm{90.79\pm 0.14}$ & $90.69\pm 0.19$ & $81.65\pm 0.02$ & $78.15\pm 1.57$\\
    	 & 10 & $\bm{89.31\pm 0.16}$ & $88.32\pm 0.16$ & $80.68\pm 0.46$ & $74.76\pm 0.73$ \\
    	 & 20 & $\bm{77.52\pm 0.63}$ & $14.15\pm 0.50$ & $76.89\pm 0.39$ & $13.86\pm 0.95$ \\
      \midrule
     \multirow{4}{*}{CIFAR-10} & 1 & $\bm{83.03\pm 1.49}$ & $83.34\pm 0.09$ & $84.34\pm 0.11$ & $64.69\pm 0.6$ \\
    	 & 5 & $\bm{74.13\pm 1.43}$ & $72.87\pm 0.76$ & $58.31\pm 0.52$ & $59.30\pm 0.56$ \\
    	 & 10 & $\bm{69.81\pm 0.87}$ & $66.75\pm 0.91$ & $43.00\pm 0.73$ & $42.92\pm 0.44$ \\
    	 & 20 & $\bm{33.61\pm 0.58}$ & $28.40\pm 0.07$ & $30.67\pm 1.35$ & $10.89\pm 0.71$ \\
      \bottomrule
  \end{tabular}}
\end{table}
{\purple\subsection{Discussion} In most cases, as shown in Tables ~\ref{tab:comparison main-banditnet} and \ref{tab:comparison linear} (also the extra experiments in App.\ref{App: experiments}),  WCE-S2BL can achieve a better policy and preserve a more stable behavior compared to baselines and the logging policy in both scenarios, linear and deep learning, if we have access to both logged datasets (known-feedback and missing-feedback). 
In KL-S2BL, which employs $\KLr(\pi_\theta\|\pi_0)$ as regularization, the  policy $\pi_\theta$ must be absolutely continuous with respect to the logging policy $\pi_0$. Thus, if the logging policy is zero at an optimal action for a given context, the learning policy cannot explore this action for the given context. 
Conversely, in WCE-S2BL, which uses the $\KLr(\pi_0\|\pi_\theta)$ for regularization, the logging policy has to be absolutely continuous with respect to the learning policy. Hence, when the logging policy is zero at an optimal action for a given context, the reverse KL regularization minimization framework provides an opportunity to explore this action for the given context and have more robust behaviour. It’s notable that by minimizing WCE-S2BL and KL-S2BL using only the logged known-feedback dataset (introduced as WCE-S2BLK and KL-S2BLK, respectively), we can observe improved performance with respect to the baselines in the most of experiments. This indicates that our regularization is also applicable even when exclusively using a logged known-feedback dataset.
More discussions are provided in App.\ref{app: discussion}.}
\section{Conclusion and future works}\label{sec:conclusion}
 We proposed an algorithm, namely, feedback-free regularized truncated IPS estimator, for Semi-supervised Batch Learning (S2BL) with logged data settings, effectively casting these kinds of problems into semi-supervised batch learning problems with logged known-feedback and missing-feedback datasets. The main take-away in feedback-free regularized batch learning is proposing regularization terms, i.e., KL divergence and reverse KL divergence between logging policy and learning policy, independent of feedback (cost) values, and also the minimization of these terms results in a tighter upper bound on true risk.  Experiments revealed that, in most cases, these algorithms can reach a learning policy performance superior to the partially unknown logging policy by exploiting the logged missing-feedback dataset and logged known-feedback dataset. In particular, the algorithm WCE-S2BL, inspired by reverse KL divergence, demonstrates superior performance over other algorithms in many cases.
 
The main limitation of this work is the assumption of access to a clean propensity score relating to the probability of an action given a context under the logging policy. We also use propensity scores in both the main objective function and the regularization term. However, we can estimate the propensity score using different methods, e.g., logistic regression~\citep{d1998propensity,weitzen2004principles}, generalized boosted models \citep{mccaffrey2004propensity}, neural networks~\citep{setoguchi2008evaluating}, parametric modeling~\citep{xie2019model} or classification and regression trees~\citep{lee2010improving,lee2011weight}. {\blue Note that, as discussed by \citet{tsiatis2006semiparametric,shi2016robust}, under the estimated propensity scores, the variance of IPS estimator reduces.} Therefore, a future line of research is to investigate how different methods of propensity score estimation can be combined with our algorithm to optimize the expected risk using logged known-feedback and missing-feedback datasets. Likewise, we believe that the idea of KL-S2BL and WCE-S2BL can be extended to semi-supervised feedback learning and using unlabeled data scenarios in reinforcement learning~\citep{konyushkova2020semi,yu2022leverage}.  As our current theoretical results hold for truncated IPS estimator, it would be interesting to investigate the effect of our proposed regularization methods on the variance of self-normalized IPS and other estimators \citep{dudik2011doubly,su2020doubly,metelli2021subgaussian,aouali2023exponential} in order to utilize the logged missing-feedback dataset. It is also interesting to apply our theoretical result (Proposition~\ref{Prop: bound KL}) to provide a lower confidence bound in pessimistic framework \cite{jin2021pessimism,jin2022policy} in terms of KL-divergence or reverse KL-divergence.

\section*{Acknowledgements} Gholamali Aminian acknowledges the support of the UKRI Prosperity Partnership Scheme (FAIR) under EPSRC Grant EP/V056883/1 and the Alan Turing Institute. 

\bibliography{refs}
\bibliographystyle{plainnat}

 \clearpage

\newpage
\appendix
\onecolumn

\section{Other Related Works}\label{App: related works}

In this section, we discuss more related works about direct methods, inverse reinforcement learning, individualized treatment effects, regularized reinforcement learning with KL divergence, semi-supervised learning, semi-supervised reinforcement learning, causal inference with missing outcomes and PAC-Bayesian approach.


\textbf{Direct Method:}  The direct method for off-policy learning from logged known-feedback datasets is based on the estimation of the cost function, followed by the application of a supervised learning algorithm to the problem ~\citep{dudik2014doubly}. However, this approach fails to generalize well, as shown by \citet{beygelzimer2009offset}. Another direct-oriented method for off-line policy learning, using the self-training approaches in semi-supervised learning, was proposed by \citet{gao2021enhancing}. A different approach based on policy optimization and boosted base learner is proposed to improve the performance in direct methods~\cite{london2023boosted}. Our approach differs from this area, as the cost function is not estimated and is based on semi-supervised batch learning with logged known-feedback and missing-feedback datasets.

\textbf{Inverse Reinforcement Learning:} 
Inverse RL, which aims to learn cost functions in a data-driven manner, has also been proposed for the setting of missing-feedback datasets in RL \citep{finn2016guided,konyushkova2020semi,abbeel2004apprenticeship}. The identifiability of cost function learning under entropy regularization is studied by \citet{cao2021identifiability}. Our work differs from this line of research, since we assume access to propensity score parameters, besides the context and action. Our logged known-feedback and missing-feedback datasets are under a fixed logging policy for all samples.

\textbf{Semi-Supervised Learning:}  There are some connections between our scenario, and semi-supervised learning~\citep{yang2021survey} approaches, including entropy minimization and pseudo-labeling.
In entropy minimization, an entropy function of predicted conditional distribution is added to the main empirical risk function, which depends on unlabeled data~ \citep{grandvalet2005semi}. The entropy function can be viewed as an entropy regularization and can lower the entropy of prediction on unlabeled data. In Pseudo-labeling, the model is trained using labeled data in a supervised manner and is also applied to unlabeled data in order to provide a pseudo label with high confidence~\citep{lee2013pseudo}. 
These pseudo-labels would be applied as inputs for another model, trained based on labeled and pseudo-label data in a supervised manner. Similar methods have been employed in the statistics literature \citep[see e.g.,][]{chakrabortty2018efficient,gronsbell2018semi}. Our work differs from the aforementioned semi-supervised learning as the logging policy biases our logged data, and the feedback for actions other than the chosen action are unavailable.
In semi-supervised learning, the label is missing for some of the data. In comparison, in our setup, the feedback is missing. 
{\blue Note that, inspired by the Pseudo-labeling algorithm in semi-supervised learning and also the work by \citet{konyushkova2020semi}, we can use a model based on the logged known-feedback dataset to assign pseudo-feedback to the logged missing-feedback dataset and then the final model is trained using the logged known-feedback dataset and logged missing-feedback dataset augmented by pseudo-feedback. Note that a regularization to reduce the variance of the IPS estimator can also be added. However, as discussed by \citet{beygelzimer2009offset}, the model fails to generalize well in the direct method where we estimate the cost function. Therefore, we do not study this method.}

\textbf{Individualized Treatment Effects:} 
The individual treatment effect aims to estimate the expected values of the squared difference between outcomes (rewards or feedback) for control and treated contexts \citep{shalit2017estimating}. In the individual treatment effect scenario, the actions are limited to two actions (treated/not treated) and the propensity scores are unknown~\citep{shalit2017estimating,johansson2016learning,alaa2017bayesian,athey2019generalized,shi2019adapting,kennedy2020towards,nie2021quasi}.{\blue Recently, the average treatment effects in semi-supervised settings (a.k.a. limited outcome data) from causal (or non-causal) inference perspective is studied by \citet{zhang2023semi,chakrabortty2022general,kallus2020role}. Our work differs from this line of works by considering larger action spaces and assuming the access to propensity scores for logged datasets.}

\textbf{Semi-Supervised Reinforcement Learning:} There are a few proposals that considered off-policy evaluation from logged data in semi-supervised learning settings from individual treatment effect \citep{sonabendsemi,cheng2021robust}. We target a different problem on off-policy learning. Recently, \citet{sonabend2020semi} and \citet{gunn2022adaptive} studied semi-supervised off-policy learning. However, an important aspect overlooked in their proposals is the regularization of the uncertainty associated with the value of the learning policy. This omission could potentially lead to sub-optimal policies in settings where specific actions have received limited exploration, a common occurrence in observational datasets \citep{levine2020offline}. 

\textbf{PAC-Bayesian Approach:} Some theoretical works for error analysis in this field are focused on the PAC-Bayesian approach (see \citet{Alquier2024user-friendly} for a comprehensive review). Relevant exmples of this line of work are e.g., \citet{london_sandler2019,sakhi2023pac,aouali2023exponential}. In particular, \citet{london_sandler2019} leveraged PAC-Bayesian theory inspired by \citet{mcallester2003simplified} to derive an upper bound on the population risk of the learning policy for truncated IPS in terms of the KL divergence, with prior and posterior distributions over the hypothesis space. Tighter generalization upper bounds via PAC-Bayesian approach is proposed by \citet{sakhi2023pac}. Meanwhile, \citet{aouali2023exponential} also applied the PAC-Bayesian approach to analyze the error of the proposed estimator. In this work, our approach is different from the PAC-Bayesian approach, and we provide an upper bound on the variance of the IPS estimator based on the KL divergence between the parameterized and logging policies.


\baselinestretch
{\blue
\section{Regularized via KL Divergence}\label{App: comp KL}
Our methods are based on regularization via KL divergence. In this discussion, we highlight the difference between our motivation for KL-divergence in comparison with other works. The KL divergence regularization with a logging policy and another learning policy is studied in off-policy reinforcement learning and batch learning~\cite{achiam2017constrained,wu2019behavior,levine2020offline,rudner2021pathologies,jaques2019way,kumar2020conservative,vieillard2020leverage}. Our work differs from this line of work by considering a counterfactual risk minimization framework. Our datasets also contain propensity scores, which are unavailable in off-policy reinforcement learning. Now, we discuss more details for the comparison with these works.
 
\subsection{Comparison with constrained policy optimization} \citet{achiam2017constrained} proposed searching for the optimal policy within a set $\Pi_{\theta}\subset \Pi$ of learning policies with parameters $\theta$. For this purpose, the optimization is done over a local neighborhood of the most recent iterate policy measured via a distance, i.e.,
$$\min_{\pi_{\theta}^k}\hat{R}(\pi_\theta,S), \quad \text{s.t.} \quad D(\pi_{\theta}^k,\pi_{\theta}^{k-1})\leq \delta, $$
where $D(\cdot,\cdot)$ is a distance measure, e.g., total variation distance. Then, by applying the Pinsker inequality, the constraint would be in terms of square root of KL divergence between successive policies during parametric policy iteration to avoid large steps. 

Our work differs from constrained policy optimization, due to,
\begin{itemize}
    \item We motivate the KL regularization (or reverse KL regularization) from variance reduction of truncated IPS estimator which is different from policy constraint approach as discussed above.
    \item Our divergences are between the learning policy at each iteration and the logging policy. However, in constrained policy optimization, the KL divergence or distance measure is computed between two successive policies during policy iterations. 
    \item In addition, we also consider the reverse KL regularization, WCE-S2BL algorithm, which is different from common divergence in constraint policy method, which is KL-divergence.
\end{itemize}

\subsection{Comparison with behavior regularized offline reinforcement learning}
Behavior Regularized Offline Reinforcement Learning (BRAC)~\cite{wu2019behavior} introduces an actor-critic framework that incorporates behavior regularization using KL divergence. This framework ensures the learning policy stays close to the logging policy while optimizing for reward. However, in our work, we also introduce KL divergence between logging policy and learning policy which is different. In addition, in our estimator of KL divergence, we are using the propensity scores which is different from BRAC approach.

\subsection{Comparison with Conservative Q-Learning}
\citet{kumar2020conservative} addressed the issue of overestimation bias in off-policy Q-learning. It proposes a novel Conservative Q-Learning loss function that incorporates KL divergence regularization, between the learning policy and a prior distribution over actions, to encourage the Q-function to be conservative at states rarely visited by the logging policy. Therefore, their motivation for KL regularization is different from variance reduction (our motivation). In addition, we utilize also KL divergence between logging policy and learning policy and vice versa.
\subsection{Comparison with KL-Regularized Reinforcement Learning from Expert Demonstrations}
The work \cite{rudner2021pathologies} focuses on using KL-regularized RL where an expert demonstration policy acts as a logging policy, influencing the learning policy direction. However, the authors show that this method, i.e., regularized via KL-divergence between the learning policy and the logging policy, can suffer from pathological training dynamics. These dynamics lead to slow learning, instability and suboptimal results. Our work differs from this work, by incorporating both KL divergence and reverse KL divergence, motivated by the variance reduction. Furthermore, we have introduced regularization through reverse KL divergence, surpassing the performance of the KL-regularized scenario.

\subsection{Comparison with Off-Policy Batch Deep Reinforcement Learning of Implicit Human Preferences in Dialog}
A  class of off-policy batch RL algorithms capable of learning effectively from a fixed batch of human interaction data, even without exploration is introduce by \cite{jaques2019way}. These algorithms leverage KL-divergence between learning policy of Q-network and prior distribution over the trajectory. Again, this work is limited in studying the KL divergence regularization and the reverse KL divergence regularization is overlooked.

}

\section{Comparison with Bayesian-CRM}\label{sec: compare London}
In this section, we compare our work with \citet{london_sandler2019} from both theoretical and algorithm perspectives. 

\subsubsection{Comparison with Theorem~\ref{Theorem: main result}}\label{app: compare theory bcrm}
We compare our Theorem~\ref{Theorem: main result} result with \citep[Theorem~1]{london_sandler2019}.
The upper bound on true risk in \citep[Theorem~1]{london_sandler2019} is derived by using the PAC-Bayesian approach, where stochastic policies with action distributions induced by distributions over hypotheses. In particular, the probability of an action $a\in\mathcal{A}$ given a context $x\in\mathcal{X}$, is equal to the probability of a random hypothesis for mapping $h:x\mapsto a$, where the probability of random hypothesis can be induced by prior or posterior distribution, $\mathbb{Q}$ or $\mathbb{P}$.

Suppose that we fix the parameter space for hypotheses set. As discussed, in \citep[Section~3.1]{london_sandler2019}, if we consider the prior distribution equal to logging policy, then KL divergence $\KLr(\mathbb{P}\|\mathbb{Q})$ can be interpreted as $\KLr(\pi_\theta\|\pi_0)$. Therefore, we can compare our upper bound in Theorem~\ref{Theorem: main result} with \citep[Theorem~1]{london_sandler2019} as follows:
\begin{itemize}
    \item Our upper bound is based on the minimum of KL divergence $D(\pi_\theta(A|X)\|\pi_0(A|X))$ and reverse KL divergence $D(\pi_0(A|X)\|\pi_\theta(A|X))$ and the upper bound in \citep[Theorem~1]{london_sandler2019} is based on reverse KL divergence only.
    \item The upper bound in \citep[Theorem~1]{london_sandler2019} has the dominating term with rate $O(\sqrt{\frac{\log(n)}{n}})$ and our upper bound contains a term with rate  $O(\frac{1}{\sqrt{n}})$ which dominates the bound.
\end{itemize}
It is worthwhile to mention that the main advantage of our bound over the PAC-Bayesian is the dependency over the reverse KL divergence, i.e $\KLr(\pi_0\|\pi_\theta)$. It helps us to define the WCE-S2BL algorithm based on $\KLr(\pi_0\|\pi_\theta)$ as regularization.

\subsubsection{Comparison with Algorithms}\label{app: compare alg bcrm}
There are two main methods proposed in \cite{london_sandler2019}.
\begin{itemize}
\item \textbf{IPS-LPR:} It is inspired by \citep[Proposition~1]{london_sandler2019} and the authors propose to minimize the following objective function,
\begin{equation}\label{eq: BCRM}
    \min_{\theta}\frac{1}{n}\sum_{i=1}^n c_i \frac{\pi_{\theta}(a_i|x_i)}{\max(p_i,\nu)}+\lambda_b \|\theta-\theta_0\|^2,
\end{equation}
where $\lambda_b$ is the hyper-parameter and $\theta_0$ is the mean of parameter under prior (logging policy). If we know the logging policy, we can compute the $\theta_0$. Otherwise, we should estimate the mean of logging policy distribution via logged known-feedback dataset. The learning policy is trained via the logged known-feedback dataset. It is worthwhile to mention that in B-CRM, \eqref{eq: BCRM}, it is assumed that the posterior variance, or variance of parameters $\theta$, is fixed to some small value, e.g., $n^{-1}$. However, in our setup, we directly, estimate the KL divergence and we have no assumption on variance of parameters. In \eqref{eq: BCRM}, after the estimation of $\theta_0$, the regularization is similar to $L_2$- regularization of model parameters and it is minimized jointly with the truncated IPS estimator via logged-known-feedback dataset to derive the parameterized logging policy.

\item \textbf{WNLL-LPR:} Another algorithm is also proposed in \cite{london_sandler2019} as WNLL-LPR where the following regularized function would be minimized,
\begin{equation}\label{eq: WNLL}
    \min_{\theta}\frac{1}{n}\sum_{i=1}^n c_i \frac{\log(\pi_{\theta}(a_i|x_i))}{\max(p_i,\nu)}+\lambda_b \|\theta-\theta_0\|^2.
\end{equation}
Note that the main objective function in WNLL-LPR is an upper bound on IPS-LPR as the feedback (cost) are non-positive, $c_i\in[-1.0]$. It's also observable that, contrasting with IPS-LPR, which can have negative values, WNLL-LPR remains positive. Therefore, WNLL-LPR is not a tight upper bound. Similarly to IPS-LPR, the regularization is minimized via the logged known-feedback dataset after setting $\theta_0$.

\end{itemize}

\section{Preliminaries}
\begin{lemma}\label{Lemma: donsker}
Suppose that $f(X)$ is $\sigma$-sub-Gaussian under distribution $Q_X$. Then, considering the difference of expectations of $f(X)$ with respect to a distribution $P_X$ and the distribution $Q_X$, the following upper bound holds:
\begin{align}
    |\mathbb{E}_{P_X}[f(X)]-\mathbb{E}_{Q_X}[f(X)]|\leq \sqrt{2\sigma ^2 \KLr(P_X\|Q_X)}
\end{align}
\end{lemma}
\begin{proof}
From the Donsker-Varadhan representation of KL divergence~\citep{polyanskiy2014lecture}, for $\gamma\in \mathbb{R}$ we have:
\begin{align}
    \KLr(P_X\|Q_X)&\geq \mathbb{E}_{P_X}[\gamma f(X)]-\log(\mathbb{E}_{Q_X}[e^{\gamma f(X)}])
    \\
    \label{eq: subgaus}
    &\geq \gamma (\mathbb{E}_{P_X}[ f(X)]-\mathbb{E}_{Q_X}[ f(X)])-\frac{\gamma^2\sigma^2}{2} ,
\end{align}
where \eqref{eq: subgaus} is the result of sub-Gaussian assumption. We have: 
\begin{align}
    \label{eq: para}
    \frac{\gamma^2\sigma^2}{2}- \gamma (\mathbb{E}_{P_X}[ f(X)]-\mathbb{E}_{Q_X}[ f(X)])+ \KLr(P_X\|Q_X)&\geq 0.
\end{align}
As in~\eqref{eq: para}, we have a quadratic in $\gamma$, which is positive and has a non-positive discriminant, then the final result holds.
\end{proof}



\section{Proofs and Details of Section~\ref{Sec: var bounds}}\label{App: Proofs var}
We first prove the following Lemma:

\begin{repproposition}{Prop: bound KL Var}\textbf{(restated)} Suppose that the importance weighted of squared cost function, i.e., $w(A,X)c^2(A,X)$, is $\sigma$-sub-Gaussian under $P_X\otimes \pi_0(A|X)$ and $P_X\otimes \pi_\theta(A|X)$, and the cost function has bounded range $[b_1,b_2]$ with $b_2\geq 0$. Then, the following upper bound holds on the variance of the importance weighted cost function:
\begin{align}
    &\operatorname{Var}\left(w(A,X)c(A,X)\right)\leq 
    \sqrt{2\sigma^2 \min(\KLr(\pi_\theta\|\pi_0),\KLr(\pi_0\|\pi_\theta)) }+b_u^2-b_l^2,
\end{align}
where $b_l=\max(b_1,0)$, $b_u=\max(|b_1|,b_2)$, $\KLr(\pi_\theta\|\pi_0)=\KLr(\pi_\theta(A|X)\|\pi_0(A|X))$ and
$\KLr(\pi_0\|\pi_\theta)=\KLr(\pi_0(A|X)\|\pi_\theta(A|X))$.
\end{repproposition}
\begin{proof}
Note that $b_l^2\leq R^2(\pi_\theta)\leq b_u^2$ where $b_l=\max(b_1,0)$ and $b_u=\max(|b_1|,b_2)$.
\begin{align}
    \operatorname{Var}\left(w(A,X)c(A,X)\right)&= \mathbb{E}_{P_X\otimes\pi_{0}(A|X)}\left[\left(w(A,X)c(A,X)\right)^2\right]- R^2(\pi_\theta)
    \\
    &\leq \mathbb{E}_{P_X\otimes\pi_{0}(A|X)}\left[\left(w(A,X)c(A,X)\right)^2\right]-b_l^2,
\end{align}
where $b_l=\max(b_1,0)$.
We need to provide an upper bound on $\mathbb{E}_{P_X\otimes\pi_{0}(A|X)}\left[\left(w(A,X)c(A,X)\right)^2\right]$. First, we have:
\begin{align}
    \label{Eq: equal form}
    \mathbb{E}_{P_X\otimes\pi_{0}(A|X)}\left[\left(w(A,X)c(A,X)\right)^2\right]
    &=\mathbb{E}_{P_X\otimes\pi_{0}(A|X)}\left[\left(\frac{\pi_\theta(A|X)}{\pi_0(A|X)}c(A,X)\right)^2\right]
    \\
    &=\mathbb{E}_{P_X\otimes\pi_\theta(A|X)}\left[\frac{\pi_\theta(A|X)}{\pi_0(A|X)}\left(c(A,X)\right)^2\right].
\end{align}
Using Lemma~\ref{Lemma: donsker} and assuming sub-Gaussianity under $P_X\otimes\pi_0(A|X)$ we have:
\begin{align}
    &\left|\mathbb{E}_{P_X\otimes\pi_\theta(A|X)}\left[\frac{\pi_\theta(A|X)}{\pi_0(A|X)}\left(c(A,X)\right)^2\right]-\mathbb{E}_{P_X\otimes\pi_0(A|X)}\left[\frac{\pi_\theta(A|X)}{\pi_0(A|X)}\left(c(A,X)\right)^2\right]\right| 
    \nn
    \\
    \label{Eq: proof donsker KL}
    &\leq\sqrt{2\sigma^2 \KLr(\pi_\theta(A|X)\|\pi_0(A|X)|P_X)},
\end{align}
and since $c(A,X)\in[b_1,b_2]$, we have:
\begin{align}
    \label{Eq: proof bounded reward}
    \mathbb{E}_{P_X\otimes\pi_0(A|X)}\left[\frac{\pi_\theta(A|X)}{\pi_0(A|X)}\left(c(A,X)\right)^2\right]=\mathbb{E}_{P_X\otimes\pi_\theta(A|X)}\left[\left(c(A,X)\right)^2\right]\leq b_u^2.
\end{align}
Considering \eqref{Eq: proof bounded reward} and \eqref{Eq: proof donsker KL}, the following result holds:
\begin{align}
    \label{Eq:  donsker KL}
    &\mathbb{E}_{P_X\otimes\pi_\theta(A|X)}\left[\frac{\pi_\theta(A|X)}{\pi_0(A|X)}\left(c(A,X)\right)^2\right]\leq \sqrt{2\sigma^2 \KLr(\pi_\theta(A|X)\|\pi_0(A|X))}+b_u^2,
\end{align}
By a similar argument and the sub-Gaussianity under $P_X\otimes\pi_\theta(A|X)$, we have:
\begin{align}
    \label{Eq:  donsker RKL}
    &\mathbb{E}_{P_X\otimes\pi_\theta(A|X)}\left[\frac{\pi_\theta(A|X)}{\pi_0(A|X)}\left(c(A,X)\right)^2\right]\leq \sqrt{2\sigma^2 \KLr(\pi_0(A|X)\|\pi_\theta(A|X))}+b_u^2,
\end{align}
And the final result holds by considering \eqref{Eq:  donsker KL}, \eqref{Eq:  donsker RKL}, and \eqref{Eq: equal form}.
\end{proof}

\begin{remark}[Uniform Coverage (Overlap) Assumption]
In the uniform coverage (overlap) assumption, it is assumed that 
\begin{equation}\label{eq: unif ass}\sup_{(a,x)\in\mathcal{A}\times \mathcal{X}}\frac{\pi_\theta(a|x)}{\pi_0(a|x)}= U_c<\infty.\end{equation}
     In this work, we assume that the importance weighted of squared cost function, i.e., $w(A,X)c^2(A,X)$, is $\sigma$-sub-Gaussian under $P_X\otimes \pi_0(A|X)$ and $P_X\otimes \pi_\theta(A|X)$. Given the constraint of a bounded reward function, the uniform coverage assumption~\eqref{eq: unif ass} implies $\sigma=\frac{U_c(b_2-b_1)}{2}$, leading to the validity of the result in Proposition~\ref{Prop: bound KL}. It's important to highlight that the sub-Gaussian assumption is a weaker assumption compared to the uniform coverage assumption. Additionally, for the sub-Gaussianity of $w(A,X)c^2(A,X)$, it is necessary, under a bounded cost function, to assume that $w(A,X)$ is itself sub-Gaussian. 
\end{remark}

\begin{repcorollary}{cor: bound KL Var}(\textbf{restated})
Suppose the cost function has a bounded range $[b_1,0]$ and a truncated IPS estimator with $\nu\in(0,1]$. Then the following upper bound holds on the variance of the truncated importance weighted cost function:
\begin{align}
   &\operatorname{Var}_{(A,X)\sim \pi_0(A|X)\otimes P_X}\left(w_\nu(A,X)c(A,X)\right)\leq 
    b_1^2\nu^{-1}\sqrt{ \min(\KLr(\pi_\theta\|\pi_0),\KLr(\pi_0\|\pi_\theta))/2 }+b_1^2,
\end{align}
where $w_\nu(A,X)=\frac{\pi_\theta(A,X)}{\max(\nu,\pi_0(A,X))}$, $\KLr(\pi_\theta\|\pi_0)=\KLr(\pi_\theta(A|X)\|\pi_0(A|X))$ and
$\KLr(\pi_0\|\pi_\theta)=\KLr(\pi_0(A|X)\|\pi_\theta(A|X))$.
\end{repcorollary}
\begin{proof}
Define $R_\nu(\pi_\theta):= \mathbb{E}_{(A,X)\sim \pi_0(A|X)\otimes P_X}\left[w_\nu(A,X)c(A,X)\right]$. Note that $0\leq R_\nu^2(\pi_\theta)\leq b_1^2$.
\begin{align}
    \operatorname{Var}\left(w(A,X)c(A,X)\right)&= \mathbb{E}_{P_X\otimes\pi_{0}(A|X)}\left[\left(w^{\nu}(A,X)c(A,X)\right)^2\right]- R_\nu^2(\pi_\theta)
    \\
    &\leq \mathbb{E}_{P_X\otimes\pi_{0}(A|X)}\left[\left(w^{\nu}(A,X)c(A,X)\right)^2\right].
\end{align}
We need to provide an upper bound on $\mathbb{E}_{P_X\otimes\pi_{0}(A|X)}\left[\left(w^{\nu}(A,X)c(A,X)\right)^2\right]$. First, we have:
\begin{align}
    \label{Eq: equal form-trun}
    \mathbb{E}_{P_X\otimes\pi_{0}(A|X)}\left[\left(w^{\nu}(A,X)c(A,X)\right)^2\right]
    &=\mathbb{E}_{P_X\otimes\pi_{0}(A|X)}\left[\left(\frac{\pi_\theta(A|X)}{\max(\pi_0(A|X),\nu)}c(A,X)\right)^2\right]
    \\
    &\leq\mathbb{E}_{P_X\otimes\pi_\theta(A|X)}\left[\frac{\pi_\theta(A|X)}{\max(\pi_0(A|X),\nu)}\left(c(A,X)\right)^2\right].
\end{align}
Using Lemma~\ref{Lemma: donsker} and the fact that the function
\[0\leq \frac{\pi_\theta(A|X)}{\max(\pi_0(A|X),\nu)}\left(c(A,X)\right)^2\leq \frac{b_1^2}{\nu},\]
is $\frac{b_1^2}{2\nu}$-sub-Gaussian under any distribution, then we have:
\begin{align}
    &\left|\mathbb{E}_{P_X\otimes\pi_\theta(A|X)}\left[\frac{\pi_\theta(A|X)}{\max(\pi_0(A|X),\nu)}\left(c(A,X)\right)^2\right]-\mathbb{E}_{P_X\otimes\pi_0(A|X)}\left[\frac{\pi_\theta(A|X)}{\max(\pi_0(A|X),\nu)}\left(c(A,X)\right)^2\right]\right| 
    \nn
    \\
    \label{Eq: proof donsker KL-trun}
    &\leq\frac{b_1^2}{\nu\sqrt{2}}\sqrt{ \KLr(\pi_\theta(A|X)\|\pi_0(A|X))},
\end{align}
and since $c(A,X)\in[b_1,0]$, we have:
\begin{align}
    \label{Eq: proof bounded reward-trun}
    \mathbb{E}_{P_X\otimes\pi_0(A|X)}\left[\frac{\pi_\theta(A|X)}{\max(\pi_0(A|X),\nu)}\left(c(A,X)\right)^2\right]=\mathbb{E}_{P_X\otimes\pi_\theta(A|X)}\left[\left(c(A,X)\right)^2\right]\leq b_1^2.
\end{align}
Considering \eqref{Eq: proof bounded reward} and \eqref{Eq: proof donsker KL}, the following result holds:
\begin{align}
    \label{Eq:  donsker KL-trunc}
    &\mathbb{E}_{P_X\otimes\pi_\theta(A|X)}\left[\frac{\pi_\theta(A|X)}{\pi_0(A|X)}\left(c(A,X)\right)^2\right]\leq b_1^2\nu^{-1}\sqrt{ \KLr(\pi_\theta(A|X)\|\pi_0(A|X))/2}+b_1^2.
\end{align}
By a similar argument and the sub-Gaussianity under $P_X\otimes\pi_\theta(A|X)$, we have:
\begin{align}
    \label{Eq:  donsker RKL-trunc}
    &\mathbb{E}_{P_X\otimes\pi_\theta(A|X)}\left[\frac{\pi_\theta(A|X)}{\pi_0(A|X)}\left(c(A,X)\right)^2\right]\leq b_1^2\nu^{-1}\sqrt{ \KLr(\pi_0(A|X)\|\pi_\theta(A|X))/2}+b_1^2.
\end{align}
And the final result holds by considering \eqref{Eq:  donsker KL}, \eqref{Eq:  donsker RKL}, and \eqref{Eq: equal form}.
\end{proof}
 We now provide a novel lower bound on the variance of the weighted cost function in the following Proposition.
\begin{proposition}(proved in App.\ref{App: Proofs var})
\label{Prop: lower bound}
Suppose that $q\leq e^{\mathbb{E}_{P_X\otimes \pi_{\theta}(A,X)}[\log(|c(A,X)|)]}$, the cost function has bounded range $[b_1,b_2]$ with $b_2\geq 0$, and consider $b_u=\max(|b_1|,b_2)$. Then, the following lower bound holds on the variance of importance weighted cost function,
\begin{align}
    \operatorname{Var}\left(w(A,X)c(A,X)\right)\geq q^2 e^{\KLr(\pi_\theta(A|X)\|\pi_0(A|X))}-b_u^2.
\end{align}
\end{proposition}

\begin{proof}
Note that $b_l^2\leq R^2(\pi_\theta)\leq b_u^2$ where $b_l=\max(b_1,0)$ and $b_u=\max(|b_1|,b_2)$.
\begin{align}
    \operatorname{Var}\left(w(A,X)c(A,X)\right)
    &= \mathbb{E}_{P_X\otimes\pi_{0}(A|X)}\left[\left(w(A,X)c(A,X)\right)^2\right]- R^2(\pi_\theta)
    \\
    &\geq \mathbb{E}_{P_X\otimes\pi_{0}(A|X)}\left[\left(w(A,X)c(A,X)\right)^2\right]-b_u^2.
\end{align}
 First, we have:
\begin{align}
    \label{Eq: equal form 1}
    \mathbb{E}_{P_X\otimes\pi_{0}(A|X)}\left[\left(w(A,X)c(A,X)\right)^2\right]
    &=\mathbb{E}_{P_X\otimes\pi_{0}(A|X)}\left[\left(\frac{\pi_\theta(A|X)}{\pi_0(A|X)}c(A,X)\right)^2\right]
    \\
    \label{eq: prop21}
    &=\mathbb{E}_{P_X\otimes\pi_\theta(A|X)}\left[\frac{\pi_\theta(A|X)}{\pi_0(A|X)}\left(c(A,X)\right)^2\right].
\end{align}
Considering \eqref{eq: prop21}, we provide a lower bound on $\mathbb{E}_{P_X\otimes\pi_\theta(A|X)}\left[\frac{\pi_\theta(A|X)}{\pi_0(A|X)}\left(c(A,X)\right)^2\right]$ as follows:
\begin{align}
    \mathbb{E}_{P_X\otimes\pi_\theta(A|X)}\left[\frac{\pi_\theta(A|X)}{\pi_0(A|X)}\left(c(A,X)\right)^2\right]
    &=\mathbb{E}_{P_X\otimes\pi_\theta(A|X)}\left[e^{\log(\frac{\pi_\theta(A|X)}{\pi_0(A|X)})+2\log(|c(A,X)|)}\right]
    \\
    &\geq e^{\mathbb{E}_{P_X\otimes\pi_\theta(A|X)}[\log(\frac{\pi_\theta(A|X)}{\pi_0(A|X)})+2\log(|c(A,X)|)]}
    \label{Eq: jensen prop2}
    \\
    &=e^{\KLr(\pi_\theta(A|X)\|\pi_0(A|X))}(e^{\mathbb{E}_{P_X\otimes\pi_\theta(A|X)}[\log(|c(A,X)|)]})^2
    \nn
    \\
    &\geq q^2 e^{\KLr(\pi_\theta(A|X)\|\pi_0(A|X))}.
    \nn
\end{align}
Where \eqref{Eq: jensen prop2} is based on Jensen-inequality for an exponential function.
\end{proof}
\begin{remark}
If we consider $c(a,x)\in[b_1,b_2]$ with $b_2\geq 0$, then we can consider $q=\max(0,b_1)$. 
\end{remark}
The lower bound on the variance of importance weights in Proposition~\ref{Prop: lower bound} can be minimized by minimizing the KL divergence or reverse KL divergence between $\pi_\theta$ and $\pi_0$.
\begin{reptheorem}{Theorem: main result}\textbf{(restated)}
Suppose the cost function takes values in $[-1,0]$. Then, for any $\delta\in(0,1)$, the following bound on the true risk of policy $\pi_\theta(A|X)$ with the truncated IPS estimator (with parameter $\nu\in(0,1]$) holds with probability at least $1-\delta$ under the distribution $P_X \otimes \pi_0(A|X)$:
\begin{align}
    R(\pi_\theta)\leq  \hat{R}_\nu(\pi_\theta,S)+ \frac{2 \log(\frac{1}{\delta})}{3\nu n}+\sqrt{\frac{ (\nu^{-1}\sqrt{ 2\min(\KLr(\pi_\theta\|\pi_0),\KLr(\pi_0\|\pi_\theta)) }+2)   \log(\frac{1}{\delta})}{n}},
\end{align}
where $\KLr(\pi_\theta\|\pi_0)=\KLr(\pi_\theta(A|X)\|\pi_0(A|X))$ and
$\KLr(\pi_0\|\pi_\theta)=\KLr(\pi_0(A|X)\|\pi_\theta(A|X))$.
\end{reptheorem}

\begin{proof}
Define $R_\nu(\pi_\theta):= \mathbb{E}_{(A,X)\sim \pi_0(A|X)\otimes P_X}\left[w_\nu(A,X)c(A,X)\right]$. Note that we have $0\leq R_\nu^2(\pi_\theta)\leq 1$ and 
\[R(\pi_\theta)\leq R_\nu(\pi_\theta).\] Let us consider $Z=\frac{\pi_\theta(A|X)}{\max(\pi_0(A|X),\nu)}c(A,X)$ and $|Z|\leq \nu^{-1}$. Then, we have:
\begin{align}
    \label{Eq: proof var upper}
    \operatorname{Var}(Z)&=\mathbb{E}_{P_X\otimes\pi_{0}(A|X)}\left[\left(\frac{\pi_\theta(A|X)}{\max(\pi_0(A|X),\nu)}c(A,X)\right)^2\right]-R_{\nu}^2(\pi_\theta)
    \\
    \nn
    &\leq \nu^{-1}\sqrt{ \frac{\min(\KLr(\pi_\theta\|\pi_0),\KLr(\pi_0\|\pi_\theta))}{2} }+1,
\end{align}
where $\KLr(\pi_\theta\|\pi_0)=\KLr(\pi_\theta(A|X)\|\pi_0(A|X))$ and
$\KLr(\pi_0\|\pi_\theta)=\KLr(\pi_0(A|X)\|\pi_\theta(A|X))$.
Using Bernstein inequality \citep{boucheron2013concentration}, we also have:
\begin{align}
    \label{Eq: Proof bern}
    Pr\left(R_{\nu}(\pi_\theta)-\hat{R}_{\nu}(\pi_\theta,S)>\epsilon \right) \leq \exp\left(\frac{-n\epsilon^2/2}{\operatorname{Var}(Z)+\epsilon \nu^{-1}/3}\right).
\end{align}
By setting $\delta=\exp\Big(\frac{-n\epsilon^2/2}{\operatorname{Var}(Z)+\epsilon \nu^{-1}/3}\Big)$ to match the upper bound in \eqref{Eq: Proof bern} and using the variance upper bound \eqref{Eq: proof var upper}, the following upper bound with probability at least $(1-\delta)$ holds under $P_X\otimes \pi_0(A|X)$:
\begin{align}
    R(\pi_\theta)&\leq R_\nu(\pi_\theta)
    \\\nn
    & \leq\hat{R}_{\nu}(\pi_\theta,S)+ \frac{\nu^{-1} \log(\frac{1}{\delta})}{3n}\\\label{eq: 1}&\quad+\sqrt{\frac{\nu^{-2} \log^2(\frac{1}{\delta})}{9n^2}+\frac{ (\nu^{-1}\sqrt{ 2\min(\KLr(\pi_\theta\|\pi_0),\KLr(\pi_0\|\pi_\theta)) }+2) \log(\frac{1}{\delta})}{n}},
\end{align}
By applying $\sqrt{x+y}\leq \sqrt{x}+\sqrt{y}$ to the last term in \eqref{eq: 1}, the final result holds.
\end{proof}

\begin{repproposition}{Prop: bound KL}\textbf{(restated)}
The following upper bound holds on the absolute difference between risks of logging policy $\pi_0(a|x)$ and the policy $\pi_\theta(a|x)$:
\begin{align}
     |R(\pi_\theta)-R(\pi_0)|\leq \min\left(\sqrt{\frac{\KLr(\pi_\theta\|\pi_0)}{2}},\sqrt{\frac{\KLr(\pi_0\|\pi_\theta)}{2}}\right),
\end{align}
where $\KLr(\pi_\theta\|\pi_0)=\KLr(\pi_\theta(A|X)\|\pi_0(A|X))$ and
$\KLr(\pi_0\|\pi_\theta)=\KLr(\pi_0(A|X)\|\pi_\theta(A|X))$.
\end{repproposition}
\begin{proof}
We have:
\begin{align}
   R(\pi_\theta)&=\mathbb{E}_{P_{X}}[\mathbb{E}_{\pi_{\theta}(A|X)}[c(A,X)]].
   \\
   R(\pi_0)&=\mathbb{E}_{P_{X}}[\mathbb{E}_{\pi_{0}(A|X)}[c(A,X)]].
\end{align}
As the cost function is bounded in $[-1,0]$, then it is $\frac{1}{2}$-sub-Gaussian under all distributions. By considering Lemma~\ref{Lemma: donsker}, the final result holds.
\end{proof}
\subsection{Proposition~\ref{Prop: bound KL Var} Comparison}
\label{App: comp kl vs chi}
Without loss of generality, let us consider $c(a,x)\in[-1,0]$. For $\sup_{(x,a)\in \mathcal{X}\times \mathcal{A}} \frac{\pi_\theta(a|x)}{\pi_0(a|x)}=\nu^{-1}<\infty$. The upper bound in Corollary~\ref{cor: bound KL Var} by considering the KL divergence $\KLr(\pi_\theta\|\pi_0)$ can be written as
\begin{align}
    \label{Eq: proof upper kl}
    &\mathbb{E}_{P_X\otimes\pi_{0}(A|X)}\left[\left(\frac{\pi_\theta(A|X)}{\pi_0(A|X)}c(A,X)\right)^2\right]\leq 
    \nu^{-1}\sqrt{ \frac{\KLr(\pi_\theta(A|X)\|\pi_0(A|X))}{2} }+1.
\end{align}
The upper bound on the second moment of importance weighted cost function in \citet[Lemma~1]{cortes2010learning} is as follows:
\begin{align}
    \label{Eq: proof Chi-square div}
    \mathbb{E}_{P_X\otimes\pi_{0}(A|X)}\left[\left(\frac{\pi_\theta(A|X)}{\pi_0(A|X)}c(A,X)\right)^2\right]\leq 
    \chi^2(\pi_\theta(A|X)\|\pi_0(A|X))+1.
\end{align}
It is shown by \cite{sason2016f} that:
\begin{align}
    \label{Eq: KL vs chi}
    D(\pi_\theta(A|X)\|\pi_0(A|X))\leq \log(\chi^2(\pi_\theta(A|X)\|\pi_0(A|X))+1).
\end{align}
Using \eqref{Eq: KL vs chi} in \eqref{Eq: proof upper kl} and comparing to \eqref{Eq: proof Chi-square div}, then for $\nu^{-1}<e^2-1$, $\exists C\in[0,\nu^{-1}]$, e.g. if $\nu^{-1}=2$ we have $C\approx1.28$, where if
$\chi^2(\pi_\theta(A|X)\|\pi_0(A|X))\geq C$, then we have:
\begin{align}
    \log(\chi^2(\pi_\theta(A|X)\|\pi_0(A|X))+1)\leq \frac{2 (\chi^2(\pi_\theta(A|X)\|\pi_0(A|X)))^2}{\nu^{-2}}.
\end{align}
Therefore, the upper bound in Proposition~\ref{Prop: bound KL Var} is tighter than \citet[Lemma~1]{cortes2010learning} for $\chi^2(\pi_\theta(A|X)\|\pi_0(A|X))\geq C$ if $\nu^{-1}<e^2-1$ and $C$ is the solution of $\log(1+x)-2x^2/\nu^{-2}=0$.

\section{Proofs and Details of Section~\ref{Sec: Semi-supervised CRM Algorithms}}\label{App: Proofs semi}

\begin{repproposition}{Prop: estimators}\textbf{(restated)}
Suppose that $\KLr(\pi_\theta(A|X)\|\pi_0(A|X))$ and the reverse $\KLr(\pi_0(A|X)\|\pi_\theta(A|X))$ are bounded.
Assuming $m_{a_i}\rightarrow \infty$ $(\forall a_i\in \mathcal{A})$, then
 $\hat{L}_{\mathrm{KL}}(\pi_\theta)$ and $\hat{L}_{\mathrm{RKL}}(\pi_\theta)$ are unbiased estimations of $\KLr(\pi_\theta(A|X)\|\pi_0(A|X))$ and $\KLr(\pi_0(A|X)\|\pi_\theta(A|X))$, respectively.
\end{repproposition}
\begin{proof}
First, we have the following decomposition:
\begin{align}
    &\KLr(\pi_\theta(A|X)\|\pi_0(A|X))
    =\sum_{i=1}^k \mathbb{E}_{P_X}\Big[(\pi_\theta(A=a_i|X)\log\Big(\frac{\pi_\theta(A=a_i|X)}{\pi_0(A=a_i|X)}\Big)\Big]
    \\
    &\KLr(\pi_0(A|X)\|\pi_\theta(A|X))=\sum_{i=1}^k \mathbb{E}_{P_X}\Big[\pi_0(A=a_i|X)\log\Big(\frac{\pi_0(A=a_i|X)}{\pi_\theta(A=a_i|X)}\Big)\Big].
\end{align}
It suffices to show that:
\begin{align}
    \label{Eq: proof decomp KL}
    &\hat{R}_{\mathrm{KL}}(\pi_\theta)\triangleq\sum_{i=1}^k \frac{1}{m_{a_i}}\sum_{(x,a_i,p)\in S_{u}\cup S} \pi_\theta(a_i|x)\log\Big(\frac{\pi_\theta(a_i|x)}{p}\Big),
    \\
    \label{Eq: proof decomp RKL}
    &\hat{R}_{\mathrm{RKL}}(\pi_\theta)\triangleq\sum_{i=1}^k \frac{1}{m_{a_i}}\sum_{(x,a_i,p)\in S_{u}\cup S} -p\log(\pi_\theta(a_i|x))+ p\log(p),
\end{align}
As we assume the divergences $\KLr(\pi_\theta(A|X)\|\pi_0(A|X))$ and $\KLr(\pi_0(A|X)\|\pi_\theta(A|X))$ are bounded, then $\mathbb{E}_{P_X}[\pi_0(a_i|X)\log(\frac{\pi_0(a_i|X)}{\pi_\theta(a_i|X)})]$ and $\mathbb{E}_{P_X}[\pi_\theta(a_i|x)\log(\frac{\pi_\theta(a_i|x)}{\pi_0(a_i|x)})]$ $\forall i\in[k]$ exist and they are bounded. 
Due to the Law of Large Numbers~\cite{hsu1947complete}, we have that:
\begin{align}
    \frac{1}{m_{a_i}}\sum_{(x,a_i,p)\in S_{u}} \pi_0(a_i|x)\log\Big(\frac{\pi_0(a_i|x)}{\pi_\theta(a_i|x)}\Big)\xrightarrow{m_{a_i}\rightarrow \infty}\mathbb{E}_{P_X}\Big[\pi_0(a_i|X)\log\Big(\frac{\pi_0(a_i|X)}{\pi_\theta(a_i|X)}\Big)\Big],
\end{align}
and
\begin{align}
    \frac{1}{m_{a_i}}\sum_{(x,a_i,p)\in S_{u}} \pi_\theta(a_i|x)\log\Big(\frac{\pi_\theta(a_i|x)}{\pi_0(a_i|x)}\Big)\xrightarrow{m_{a_i}\rightarrow \infty}\mathbb{E}_{P_X}\Big[\pi_\theta(a_i|x)\log\Big(\frac{\pi_\theta(a_i|x)}{\pi_0(a_i|x)}\Big)\Big].
\end{align}
By considering \eqref{Eq: proof decomp KL}, \eqref{Eq: proof decomp RKL} and $m_{a_i}\rightarrow \infty$, $\forall i\in[k]$, the final results hold.
\end{proof}
We also provide an upper bound on the estimation error of the proposed estimator in Proposition~\ref{Prop: estimators}. Let us define \[f_{\KLr}(x,a):=\pi_\theta(A=a|X=x)\log\Big(\frac{\pi_\theta(A=a|X=x)}{\pi_0(A=a|X=x)}\Big),\] and \[g_{\mathrm{RKL}}(x,a):=\pi_0(A=a|X=x)\log\Big(\frac{\pi_0(A=a|X=x)}{\pi_\theta(A=a|X=x)}\Big).\]
Note that 
\[\mathbb{E}_{P_X}[f_{\KLr}(X,a_i)]=\KLr(\pi_\theta(A=a_i|X)\|\pi_0(A=a_i|X)),\]
and \[\mathbb{E}_{P_X}[g_{\mathrm{RKL}}(X,a_i)]=\KLr(\pi_0(A=a_i|X)\|\pi_\theta(A=a_i|X)).\]
\begin{proposition}\label{Prop: error analysis}
     Assume that $|f_{\KLr}(x,a)|\leq B$ and $|g_{\mathrm{RKL}}(x,a)|\leq C$ for all $x\in\mathcal{X}$ and $a\in \mathcal{A}$. Then, the following upper bounds hold on error of estimators of KL divergence and reverse KL divergence in Proposition~\ref{Prop: estimators}, under distribution $P_X$ with probability at least $1-\delta$ for $\delta\in(0,1]$,
    \begin{equation}
        \Big|\KLr(\pi_\theta(A|X)\|\pi_0(A|X))-\hat{R}_{\mathrm{KL}}(\pi_\theta)\Big|\leq B\sqrt{2\log(k/\delta)}\sum_{i=1}^k \sqrt{\frac{1}{m_{a_i}}},
    \end{equation}
    and similarly, we have 
    \begin{equation}
   \Big| \KLr(\pi_0(A|X)\|\pi_\theta(A|X))-\hat{R}_{\mathrm{RKL}}(\pi_\theta)\Big|\leq C\sqrt{2\log(k/\delta)}\sum_{i=1}^k \sqrt{\frac{1}{m_{a_i}}}.
    \end{equation}
\end{proposition}
\begin{proof}
    From Hoeffding’s inequality~\cite{boucheron2013concentration}, for each action $a_i\in\mathcal{A}$, the following upper bound holds with probability at least $(1-\delta)$ under distribution $P_X$,
    \begin{equation}
      \Big | \mathbb{E}_{P_X}[f_{\KLr}(X,a_i)]-\frac{1}{m_{a_i}}\sum_{j=1}^{m_{a_i}}f_{\KLr}(x_j,a_i)\Big|\leq B\sqrt{\frac{2\log(1/\delta)}{m_{a_i}}},
    \end{equation}
    and similarly 
      \begin{equation}
      \Big | \mathbb{E}_{P_X}[g_{\mathrm{RKL}}(X,a_i)]-\frac{1}{m_{a_i}}\sum_{j=1}^{m_{a_i}}g_{\mathrm{RKL}}(x_j,a_i)\Big|\leq C\sqrt{\frac{2\log(1/\delta)}{m_{a_i}}}.
    \end{equation}
    Using the Union bound~\cite{vershynin2018high} and considering $|\mathcal{A}|=k$, the following upper bound holds on the estimation error of the proposed estimator in Proposition~\ref{Prop: estimators} under distribution $P_X$ with probability at least $(1-k\delta)$ for $\delta\in(0,1/k]$,
    \begin{equation}
    \begin{split}
     &\Big|\KLr(\pi_\theta(A|X)\|\pi_0(A|X))-\hat{R}_{\mathrm{KL}}(\pi_\theta)\Big|\\
     &\leq \sum_{i=1}^k \Big|\KLr(\pi_\theta(A=a_i|X)\|\pi_0(A=a_i|X))-\frac{1}{m_{a_i}}\sum_{j=1}^{m_{a_i}}f_{\KLr}(x_j,a_i)\Big|\\
     &\leq B\sqrt{2\log(1/\delta)}\sum_{i=1}^k \sqrt{\frac{1}{m_{a_i}}},
         \end{split}
    \end{equation}
    and similarly, we have, 
    \begin{equation}
    \begin{split}
   &\Big| \KLr(\pi_0(A|X)\|\pi_\theta(A|X))-\hat{R}_{\mathrm{RKL}}(\pi_\theta)\Big|\\
   &\leq  \sum_{i=1}^k \Big|\KLr(\pi_0(A=a_i|X)\|\pi_\theta(A=a_i|X))-\frac{1}{m_{a_i}}\sum_{j=1}^{m_{a_i}}g_{\mathrm{RKL}}(x_j,a_i)\Big|\\
   &\leq C\sqrt{2\log(1/\delta)}\sum_{i=1}^k \sqrt{\frac{1}{m_{a_i}}}.
    \end{split}
    \end{equation}
    The final result holds by consider the scaling $\delta/k$.
\end{proof}

\begin{remark}
    Suppose that we have equal number of samples per action in the set $S\cup S_u$, i.e., $m_{a_i}=\frac{m+n}{k}$. Then the estimation error of KL divergence and reverse KL divergence is,
     \begin{equation}
        \Big|\KLr(\pi_\theta(A|X)\|\pi_0(A|X))-\hat{R}_{\mathrm{KL}}(\pi_\theta)\Big|\leq B\sqrt{2\log(k/\delta)}k \sqrt{\frac{k}{m+n}},
    \end{equation}
    and
    \begin{equation}
   \Big| \KLr(\pi_0(A|X)\|\pi_\theta(A|X))-\hat{R}_{\mathrm{RKL}}(\pi_\theta)\Big|\leq C\sqrt{2\log(k/\delta)}k \sqrt{\frac{k}{m+n}}.
    \end{equation}
    Therefore, for fix number of actions, $k$, if we increase the number of unlabeled samples, i.e., $m$, the estimation error decreases.
\end{remark}
\subsection{Regret Upper Bound}\label{app: regret upper bound}

Using our current theoretical results, we can derive an upper bound on regret, i.e., $|R(\pi^\star_\theta)- R(\pi^r_{\theta})|$, where the solution to our KL-regularized risk minimization is denoted by $\pi^r_{\theta}$.

\begin{theorem}\label{thm: regret bound}
    Suppose that the cost function takes values in $[-1,0]$. Then for any $\delta \in (0,1)$, the following bound on the regret of  $\pi^r_{\theta}(A|X)$ with the truncated IPS estimator holds with probability at least $(1-\delta)$ under distribution $P_X\otimes \pi_0(A|X)$,

     \[ | R(\pi^\star_\theta)- R(\pi^r_{ \theta }) |  \leq \hat{R}_\nu ( \pi^\star_\theta , S) - \hat{R}_\nu (\pi_{\theta}^r , S)  + \frac{ 4 \log(2/\delta)}{ 3\nu n } + \sqrt{ \frac{2\log( 2/\delta ) M }{ n }}, \]

where $M=\min\big(\mathrm{KL}(\pi^\star_\theta\|\pi_0),\mathrm{KL}(\pi_0\|\pi^\star_\theta)\big)+\min\big(\mathrm{KL}(\pi^r_{\theta}\|\pi_0),\mathrm{KL}(\pi_0\|\pi^r_{\theta})\big)$.
\end{theorem}

\begin{proof}
    
In Theorem 1, our upper bound holds on true risks of any learning policy $\pi_\theta(A|X)$. Therefore, it also holds for optimal $\pi^\star_{\theta}$ and  Therefore, using the following decomposition, we have
\[R(\pi^\star_\theta)- R(\pi^r_{\theta}) = I_1 + I_2 +I_3,\]
where 
\[\begin{split}   
I_1&:= R ( \pi^\star_\theta) - \hat{R}_\nu (\pi^\star_\theta , S), 
\\
 I_2&:= \hat{R}_\nu (\pi^\star_\theta,S)- \hat{R}_\nu (\pi^r_{\theta},S), \\
 I_3&:= \hat{R}_\nu (\pi^r_{\theta},S) -R (\pi^r_{\theta}).
\end{split}
\]
Therefore, we have,
\[|R(\pi^\star_\theta)- R(\pi^r_{\theta})| \leq |I_1| + I_2 + |I_3|,\]
where we can apply Theorem 1 on $|I_1|$ and $|I_3|$, to provide an upper bound. Subsequently, the following upper bound holds on regret of our regularized algorithm with probability at least $1-\delta$ for $\delta\in(0,1)$,

\[ | R(\pi^\star_\theta)- R(\pi^r_{ \theta }) |  \leq \hat{R}_\nu ( \pi^\star_\theta , S) - \hat{R}_\nu (\pi_{\theta}^r , S)  + \frac{ 4 \log(2/\delta)}{ 3\nu n } + \sqrt{ \frac{2\log( 2/\delta ) M }{ n }}, \]

where $M=\min\big(\mathrm{KL}(\pi^\star_\theta\|\pi_0),\mathrm{KL}(\pi_0\|\pi^\star_\theta)\big)+\min\big(\mathrm{KL}(\pi^r_{\theta}\|\pi_0),\mathrm{KL}(\pi_0\|\pi^r_{\theta})\big)$. Therefore, our results can be applied to provide an upper bound on the regret of our algorithm. 
\end{proof}

We can observe from Theorem~\ref{thm: regret bound}, where the upper bound on the regret depends on KL divergence or reverse KL divergence between the pair $(\pi^\star_\theta,\pi_0)$ and $(\pi_0,\pi_\theta^r)$.

\section{True Risk Regularization}\label{App: true risk minimizer}

We can choose the KL divergence instead of the square root of the KL divergence as a regularizer for IPS estimator minimization. In this section, we study the true risk regularization using KL divergence  $\KLr(\pi_\theta(A|X)\|\pi_0(A|X))$, as follows:
\begin{align}\label{Eq: true reg minimize}
   \min_{\pi_\theta} R(\pi_\theta) + \lambda \KLr(\pi_\theta(A|X)\|\pi_0(A|X)), \quad \lambda\geq 0.
\end{align}
It is possible to provide the optimal solution to regularized minimization \eqref{Eq: true reg minimize}.
\begin{theorem}
Considering the true risk minimization with KL divergence regularization,
\begin{align}
   \min_{\pi_\theta} R(\pi_\theta) + \lambda \KLr(\pi_\theta(A|X)\|\pi_0(A|X)), \quad \lambda\geq 0,
\end{align}
the optimal learning policy is:
\begin{align}
    \pi_\theta^\star(A=a|X=x)=\frac{\pi_0(A=a|X=x)e^{-\frac{1}{\lambda }c(a,x)}}{\mathbb{E}_{\pi_0}[e^{-\frac{1}{\lambda }c(a,x)}]}.
\end{align}
\end{theorem}
\begin{proof}
The minimization problem \eqref{Eq: true reg minimize} can be written as follows:
\begin{align}
   \min_{\pi_\theta} \mathbb{E}_{P_X}[\mathbb{E}_{\pi_\theta(A|X)}[c(A,X)]] + \lambda \KLr(\pi_\theta(A|X)\|\pi_0(A|X)), \quad \lambda\geq 0.
\end{align}
Using the same approach used by \citet{zhang2006information,aminian2021exact} and considering $\frac{1}{\lambda}$ as the inverse temperature, the final result holds.
\end{proof}
The optimal learning policy under KL divergence regularization, i.e., 
\begin{align}
\pi_\theta^\star(A=a|X=x)=\frac{\pi_0(A=a|X=x)e^{-\frac{1}{\lambda }c(a,x)}}{\mathbb{E}_{\pi_0}[e^{-\frac{1}{\lambda }c(a,x)}]},
\end{align}
provides the following insights:
\begin{itemize}
\item The optimal learning policy, $\pi_\theta^\star(A|X)$, is a stochastic policy.
\item The optimal learning policy is invariant with respect to constant shifts in the cost function.
\item For asymptotic condition, i.e., $\lambda\rightarrow 0$, the optimal learning policy will be deterministic policy.
\end{itemize}

\section{Experiments}\label{App: experiments}

\subsection{Setup Details}
{\blue In our experiments, we use the following image classification datasets, Fashion-MNIST (FMNIST)~\citep{xiao2017/online},  EMNIST \citep{cohen2017emnist}, CIFAR-10 and CIFAR-100~\citep{krizhevsky2009learning}.} We also use KuaiRec dataset as a real-world example, details explained in section~\ref{sec:real_world}. A summary of the statistics of these datasets is provided in \cref{tab:stats}. We use a combination of manual and automatic hyper-parameter tuning for the learning rate values and regularization coefficient $\lambda$. To be more specific, for the deep model we manually test different hyper-parameters for $\tau=1, 5$ and use them to set search intervals for other values of $\tau$ and all values of $\tau$ for the deep model. For automatic search we use optuna library. We train each model by $ 120$ and $60$ epochs for deep and linear models respectively and use a learning rate multiplier of $0.5$ in every $25$ epochs. Inspired by BanditNet experiments in \citet{joachims2018deep}, for the CIFAR-10 dataset, we ignore samples with less than $\nu=0.001$ propensity score, while for the FMNIST dataset after grid search, we consider $\nu=0.001$ as the truncation parameter. Table~\ref{tab:experiment setup} illustrates the experiment settings (Real-world dataset settings are separately in section~\ref{sec:real_world})

\begin{table}[t]{\blue
\caption{Statistics of the datasets used in our experiments.}
\label{tab:stats}
\vskip 0.15in
\begin{center}
\begin{small}
\begin{sc}
\centering
\begin{tabular}{lcccr}
\toprule
Data set & training samples & test samples  & number of actions & Dimension \\
\bottomrule
FMNIST & 60000 & 10000 & 10 & $28 \times 28$\\
EMNIST    & 60000 & 10000 & 10 & $28 \times 28$ \\
CIFAR-10    & 50000 & 10000 & 10 & $32 \times 32 \times 3$ \\
CIFAR100    & 50000 & 10000 & 100 & $32 \times 32 \times 3$ \\
Kuairec    & 12,530,806 & 4,676,570 & 10,728 & $1555$ \\
\bottomrule
\end{tabular}
\end{sc}
\end{small}
\end{center}}
\vskip -0.1in
\end{table}
\subsection{Deep Model Architecture}
We use two simple versions of ResNet architecture. For ResNet-v1 we use a single residual layer in each of the four blocks. For ResNet-v2 we use two residual layers in each of the blocks.
\begin{table}[tbh!]
    \centering
     \caption{Experiment setup details for the softmax policy with deep learning}
    \label{tab:experiment setup}
    \begin{center}
    \begin{tabular}{ccc}
    \toprule
         & Deep model & Linear model \\
         \midrule
       Optimizer  & SGD  & SGD \\
      Truncation parameter $(\nu)$   & $0.001$ & $0.001$ \\
      Network & ResNet-v2 & Linear\\
       Learning rate  & 0.005  & 0.0005 \\
       Max epochs ($M$) & 120 & 60\\
       Batch size & 128 & 128\\
        \bottomrule
    \end{tabular}
        \end{center}
\end{table}
\subsection{Bandit Dataset Generation}
We create a bandit dataset consisting of samples $(x, y, a, p, c)$ where $x$ is the context, $y$ is the true label (optimal action), $a$ is the logging policy's action, $p$ is the propensity score, and $c$ is the feedback (cost) of the action. To do so, starting with a labeled dataset (CIFAR-10, CIFAR-100, EMNIST and FMNIST in our experiments) containing only the pair $(x, y)$ for each sample, we first train a logging policy using the true labels $y$, with fully supervised feedback. For each context $x$ in the labeled dataset, we sample an action $a$ and compute propensity score $p$ from the trained logging policy according to the softmax output of the model, and compute the cost value $c$. Hence the tuple $(x, y, a, p, c)$ is created. In order to decrease the performance of the logging policy in a controlled manner, for each $\tau \in \{1, 5, 10, 20\}$ we train a logging policy with temperature $\tau$. During dataset generation, we sample from the logging policy with temperature $\tau=1$. So the trained logging policy's performance decreases as $\tau$ increases. For each $\rho \in \{1, 0.5, 0.2, 0.1, 0.02\}$, we randomly select $\rho$ proportion of the samples and remove the feedback from other samples. \\
Therefore for each labeled dataset, we create $20 = 4 \times 5$ bandit datasets for different values of $\tau$ and $\rho$. For a fair comparison between different methods, we create and store these datasets once, and apply the models on the same dataset for each setting. \\
{ \blue 
For CIFAR-10, FashionMNIST, and EMNIST datasets in linear model, we flatten the image to get a $3072$, $784$, and $784$ dimensional feature vector respectively. For CIFAR-100 we use ResNet-50 pretrained features. We use this vector as the context $x$.
 \\
 The architecture of the logging policy is ResNet-v1 for CIFAR-10, FMNIST, and EMNIST. For CIFAR-100 we use ResNet-v2.
 \\
 Note that for linear experiments on CIFAR-10 and FMNIST, we trained a deep logging policy. However for EMNIST and CIFAR-100, we used a linear model for the logging policy. Table~\ref{tab:linear_log_policy} shows a summary of features and models in a linear setting. The reason behind the different settings is to observe the difference in performance when the logging policy is of different architectures. We also carried out experiments on CIFAR-10 with linear logging policy, explained in section~\ref{sec:cifar_pretrained}.
 }
\begin{table}[tbh!]{\blue
        \caption{Summary of models and features in linear experiments}
            \label{tab:linear_log_policy}}
\begin{center}
    
    \begin{tabular}{cccc}
    \toprule
         & Logging policy & Trained Policy & Features \\
         \midrule
       FMNIST  & deep  & linear & raw \\
      CIFAR-10  & linear/deep  & linear & raw \\
      EMNIST & linear  & linear & raw \\
       CIFAR-100  & linear  & linear & pre-trained \\
        \bottomrule
    \end{tabular}
    \end{center}

\end{table}
\subsection{Baselines}\label{App: baselines}
We consider two baselines in our experiments for linear and deep setup.

\textbf{Linear Model:} In this setup, as we are focused on truncated IPS estimator, therefore we choose the Bayesian-CRM (B-CRM) method based on \citet[Proposition~1]{london_sandler2019} introduced in \eqref{eq: BCRM}. For B-CRM as our baseline, we estimate $\mu_0$ using logged known-feedback dataset. 
%


\textbf{Deep Model:} In this setup, we consider the BanditNet \cite{joachims2018deep} as baseline. Note that, in BanditNet, instead of an IPS estimator, we have a self-normalized IPS (SNIPS) estimator. In particular, the SNIPS estimator is defined as
\begin{equation}\label{eq: banditnet}
   \text{SNIPS}:= \frac{\sum_{i=1}^n c_i\frac{\pi_\theta(a_i|x_i)}{\pi_0(a_i|x_i)}}{\sum_{i=1}^n \frac{\pi_\theta(a_i|x_i)}{\pi_0(a_i|x_i)}}.
\end{equation}
However, the SNIPS estimator in \eqref{eq: banditnet} can not be optimized by SGD and \cite{joachims2018deep} proposed BanditNet as a constraint optimization version of \eqref{eq: banditnet} which can be optimized by SGD.
\subsection{Results}\label{app: more results}
For CIFAR-10 and FMNIST, in Tables~\ref{tab:comparison main-banditnet-full} and \ref{tab:comparison main-linear-full}, we compare the performance of all proposed algorithms, WCE-S2BL, KL-S2BL, WCE-S2BLK and KL-S2BLK in both deep and linear models with the baselines, BanditNet in deep model and Bayesian CRM in linear model for $\tau\in\{1,5,10,20\} $ and $\rho\in\{0.02,0.1,0.2,0.5,1 \}$. {\blue Similarly, for CIFAR-100 and EMNIST the results are presented in Tables~\ref{tab:comparison main-linear-EM-Cifar-100} and \ref{tab:comparison main-banditnet-full-emn-cifar-100}.} Note that, for $\rho=1$, where we have access to all logged known-feedback dataset, WCE-S2BL and KL-S2BL are the same as WCE-S2BLK and KL-S2BLK, respectively. 

\begin{table*}[htb!]
    \centering
  \caption{Comparison of different algorithms WCE-S2BL, KL-S2BL, WCE-S2BLK, KL-S2BLK and BanditNet deterministic policy accuracy for FMNIST and CIFAR-10 with deep model setup and different qualities of logging policy ($\tau\in\{1,5,10,20\}$) and proportions of labeled data ($\rho\in \{0.02,0.1,0.2,0.5,1\}$). }
  \label{tab:comparison main-banditnet-full}
  \begin{center}
  \resizebox{0.9\textwidth}{!}{
	\begin{tabular}{ccccccccc}
 \toprule
    	Dataset & $\tau$ & $\rho$ & \textbf{WCE-S2BL} & \textbf{KL-S2BL} & \textbf{WCE-S2BLK} & \textbf{KL-S2BLK} & \textbf{BanditNet} & \textbf{Logging Policy}\\
    	\midrule
    	\multirow{4}{*}{FMNIST} & \multirow{2}{*}{1} & 0.02 & \bm{$93.12\pm 0.16$} & $91.79\pm 0.16$ & $78.66\pm 0.90$ & $61.46\pm 9.97$ & $78.64\pm 1.97$ & \multirow{5}{*}{$91.73$}\\
    	 &  & 0.1 & \bm{$93.26\pm 0.05$} & $91.73\pm 0.08$ & $85.83\pm 0.85$ & $77.75\pm 9.10$ & $84.64\pm 4.24$ & \\
      &  & 0.2 & \bm{$93.16\pm 0.18$} & $92.04\pm 0.13$ & $82.76\pm 4.45$ & $87.72\pm 0.53$ & $89.60\pm 0.49$ & \\
      &  & 0.5 & \bm{$93.19\pm 0.21$} & $91.94\pm 0.04$ & $88.72\pm 0.37$ & $86.30\pm 1.43$ & $91.59\pm 0.03$ & \\
      &  & 1 & $93.10\pm 0.15$ & $92.48\pm 0.6$ & $-$ & $-$ & \bm{$93.54\pm 0.03$} & \\
      \cmidrule(lr{0.5em}){2-9}
      & \multirow{2}{*}{5} & 0.02 & \bm{$90.99\pm 0.09$} & $83.54\pm 0.66$ & $81.67\pm 0.36$ & $34.27\pm 27.64$ & $47.11\pm 12.51$ & \multirow{5}{*}{$53.97$}\\
    	 &  & 0.1 & \bm{$90.79\pm 0.14$} & $81.65\pm 0.02$ & $87.93\pm 0.07$ & $73.48\pm 13.26$ & $86.73\pm 0.63$ & \\
      &  & 0.2 & \bm{$91.43\pm 0.07$} & $82.71\pm 0.59$ & $89.47\pm 0.06$ & $88.94\pm 0.34$ & $89.17\pm 0.26$ & \\
      &  & 0.5 & \bm{$91.74\pm 0.04$} & $88.36\pm 0.15$ & $89.18\pm 0.47$ & $90.45\pm 0.12$ & $90.42\pm 0.56$ & \\
      &  & 1 & $91.41\pm 0.16$ & $92.42\pm 0.12$ & $-$ & $-$ & \bm{$92.65\pm 0.04$} & \\
      \cmidrule(lr{0.5em}){2-9}
      & \multirow{2}{*}{10} & 0.02 & \bm{$89.35\pm 0.15$} & $69.94\pm 0.60$ & $77.82\pm 0.73$ & $45.18\pm 19.82$ & $23.52\pm 3.15$ & \multirow{5}{*}{$20.72$}\\
    	 &  & 0.1 & \bm{$89.31\pm 0.16$} & $80.68\pm 0.46$ & $85.55\pm 0.39$ & $80.54\pm 6.88$ & $82.96\pm 3.03$ & \\
      &  & 0.2 & \bm{$89.47\pm 0.3$} & $79.45\pm 0.75$ & $88.31\pm 0.14$ & $67.53\pm 2.06$ & $88.35\pm 0.45$ & \\
      &  & 0.5 & $90.05\pm 0.13$ & $89.38\pm 0.13$ & $89.81\pm 0.23$ & $89.63\pm 0.98$ & \bm{$90.44\pm 0.08$} & \\
      &  & 1 & $91.00\pm 0.19$ & $91.45\pm 0.17$ & $-$ & $-$ & \bm{$92.21\pm 0.07$} & \\
      \cmidrule(lr{0.5em}){2-9}
    	 & \multirow{2}{*}{20} & 0.02 & \bm{$52.60\pm 1.36$} & $45.20\pm 4.74$ & $41.69\pm 2.19$ & $22.23\pm 3.30$ & $44.04\pm 7.50$ & \multirow{5}{*}{$10.54$} \\
    	 & & 0.1 & \bm{$77.52\pm 0.63$} & $76.89\pm 0.39$ & $76.29\pm 0.48$ & $71.28\pm 5.14$ & $75.64\pm 1.65$ & \\
      &  & 0.2 & \bm{$84.02\pm 0.28$} & $82.25\pm 0.60$ & $82.85\pm 0.94$ & $82.69\pm 0.41$ & $80.63\pm 0.13$ & \\
      &  & 0.5 & $86.83\pm 0.22$ & $87.48\pm 0.26$ & $86.51\pm 0.32$ & $87.77\pm 0.17$ & \bm{$87.61\pm 0.10$} & \\
      &  & 1 & $87.05\pm 0.01$ & $89.11\pm 0.10$ & $-$ & $-$ & \bm{$89.03\pm 0.16$} & \\
      \midrule
     \multirow{4}{*}{CIFAR-10} &  \multirow{2}{*}{1} & 0.02 & \bm{$85.01\pm 0.37$} & $84.6\pm 0.65$ & $17.12\pm 0.97$ & $21.63\pm 1.44$ & $27.39\pm 3.47$ & \multirow{5}{*}{$79.77$} \\
    	 & & 0.1 & \bm{$83.03\pm 1.49$} & $84.34\pm 0.11$ & $51.84\pm 0.92$ & $46.24\pm 0.41$ & $52.78\pm 0.56$ &  \\
      & & 0.2 & \bm{$85.06\pm 0.32$} & $85.53\pm 0.56$ & $58.04\pm 5.47$ & $54.12\pm 0.51$ & $67.96\pm 0.62$ &  \\
      & & 0.5 & \bm{$84.79\pm 0.4$} & $84.5\pm 0.09$ & $79.23\pm 0.30$ & $78.74\pm 0.56$ & $71.36\pm 1.91$ & \\
      & & 1 & $84.63\pm 0.38$ & $84.25\pm 0.45$ & $-$ & $-$ & \bm{$86.82\pm 0.87$} & \\
       \cmidrule(lr{0.5em}){2-9}
    	 & \multirow{2}{*}{5} & 0.02 & \bm{$73.57\pm 0.36$} & $51.14\pm 2.25$ & $17.12\pm 0.97$ & $21.63\pm 1.44$ & $15.81\pm 5.12$ & \multirow{5}{*}{$53.97$} \\
    	 & & 0.1 & \bm{$74.13\pm 1.43$} & $58.31\pm 0.52$ & $54.75\pm 0.39$ & $33.21\pm 0.88$ & $24.68\pm 3.74$ & \\
      & & 0.2 & \bm{$76.96\pm 0.35$} & $63.19\pm 0.51$ & $62.98\pm 0.81$ & $46.35\pm 0.10$ & $30.03\pm 12.75$ & \\
      & & 0.5 & \bm{$76.46\pm 0.7$} & $64.24\pm 2.09$ & $70.50\pm 0.86$ & $55.92\pm 0.66$ & $58.34\pm 8.69$ & \\
      & & 1 & \bm{$77.53\pm 1.19$} & $69.53\pm 1.09$ & $-$ & $-$ & $70.12\pm 6.89$ & \\
       \cmidrule(lr{0.5em}){2-9}
    	 & \multirow{2}{*}{10} & 0.02 & \bm{$65.67\pm 1.06$} & $37.8\pm 0.85$ & $32.61\pm 1.14$ & $20.66\pm 5.74$ & $13.78\pm 1.99$ & \multirow{5}{*}{$43.45$} \\
    	 & & 0.1 & \bm{$69.81\pm 0.87$} & $43.00\pm 0.73$ & $51.15\pm 0.64$ & $35.87\pm 1.11$ & $21.19\pm 3.35$ & \\
      & & 0.2 & \bm{$69.4\pm 0.47$} & $48.44\pm 0.26$ & $55.38\pm 3.63$ & $44.60\pm 0.19$ & $50.38\pm 0.55$ & \\
      & & 0.5 & \bm{$75.08\pm 0.18$} & $64.39\pm 0.05$ & $71.90\pm 0.14$ & $16.19\pm 0.99$ & $68.92\pm 0.68$ & \\
      & & 1 & $75.58\pm 0.29$ & \bm{$79.82\pm 0.36$} & $-$ & $-$ & $78.8\pm 0.53$ & \\
       \cmidrule(lr{0.5em}){2-9}
    	 & \multirow{2}{*}{20} & 0.02 & \bm{$26.24\pm 1.42$} & $15.09\pm 1.5$ & $16.46\pm 1.77$ & $12.56\pm 2.01$ & $13.25\pm 1.36$ & \multirow{5}{*}{$20.72$} \\
    	 & & 0.1 & \bm{$33.61\pm 0.58$} & $30.67\pm 1.35$ & $27.38\pm 2.44$ & $27.74\pm 8.23$ & $21.12\pm 1.01$ &  \\
      & & 0.2 & $34.49\pm 4.01$ & \bm{$36.95\pm 0.77$} & $32.91\pm 6.95$ & $34.27\pm 2.55$ & $32.69\pm 2.17$ & \\
      & & 0.5 & $46.95\pm 0.89$ & \bm{$50.12\pm 4.43$} & $47.69\pm 0.63$ & $41.45\pm 9.93$ & $36.79\pm 2.78$ & \\
      & & 1 & $47.68\pm 3.03$ & \bm{$64.34\pm 0.85$} & $-$ & $-$ & $55.27\pm 3.39$ & \\
      \bottomrule
  \end{tabular}}
   \end{center}
\end{table*}

\begin{table*}[htp!]

  \caption{Comparison of different algorithms WCE-S2BL, KL-S2BL, WCE-S2BLK, KL-S2BLK and BanditNet deterministic policy accuracy for EMNIST and CIFAR-100 with deep model setup and different qualities of logging policy ($\tau\in\{10,20\}$ for EMNIST and $\tau\in\{1, 5, 10\}$ for CIFAR-100) and proportions of labeled data ($\rho\in \{0.02,0.1,0.2,0.5,1\}$). }
  \label{tab:comparison main-banditnet-full-emn-cifar-100}
  \centering
  \resizebox{0.9\linewidth}{!}{
	\begin{tabular}{ccccccccc}
	\\
 \toprule
    	Dataset & $\tau$ & $\rho$ & \textbf{WCE-S2BL} & \textbf{KL-S2BL} & \textbf{WCE-S2BLK} & \textbf{KL-S2BLK} & \textbf{BanditNet} & \textbf{Logging Policy}\\
    	\midrule
 \multirow{2}{*}{EMNIST} & \multirow{2}{*}{10} & 0.02 & \bm{$98.77\pm0.06$} & $93.76\pm0.46$ & $93.25\pm0.65$ & $67.46\pm40.63$ & $95.44\pm0.12$ & \multirow{5}{*}{$51.26$} \\
 &  & 0.1 & \bm{$98.75\pm0.01$} & $98.14\pm0.15$ & $96.61\pm0.25$ & $98.60\pm0.10$ & $98.46\pm0.01$ &  \\
 &  & 0.2 & $98.81\pm0.02$ & $98.49\pm0.01$ & $98.13\pm0.06$ & $98.66\pm0.04$ & \bm{$99.11\pm0.01$} &  \\
 &  & 0.5 & $99.16\pm0.04$ & $99.03\pm0.00$ & $99.17\pm0.01$ & $99.09\pm0.02$ & \bm{$99.25\pm0.07$} &  \\
 &  & 1 & $99.38\pm0.05$ & $99.39\pm0.02$ & $-$ & $-$ & \bm{$99.46\pm0.02$} &  \\
\cmidrule(lr{0.5em}){2-9} & \multirow{2}{*}{20} & 0.02 & \bm{$96.54\pm0.06$}
& $84.98\pm3.04
$ & $79.49\pm1.87
$ & $93.47\pm0.50
$ & $89.50\pm5.04
$ & \multirow{5}{*}{$25.58$} \\
 &  & 0.1 & $97.79\pm0.14$ & $97.83\pm0.02$ & $97.88\pm0.11$ & \bm{$98.31.0.04$} & $98.14\pm0.12$ &  \\
 &  & 0.2 & $98.49\pm0.04$ & $98.33\pm0.02$ & $98.50\pm0.05$ & $98.42\pm0.07$ & \bm{$98.59\pm0.05$} &  \\
 &  & 0.5 & $98.83\pm0.08$ & $98.81\pm0.03$ & $98.79\pm0.05$ & $98.83\pm0.11$ & \bm{$99.07\pm0.04$} &  \\
 &  & 1 & $99.08\pm0.03$ & \bm{$99.35\pm0.03$} & $-$ & $-$ & $99.16\pm0.02$ &  \\
      \midrule
\multirow{3}{*}{CIFAR-100} & \multirow{2}{*}{1} & 0.02 & \bm{$38.60\pm0.23$} & $15.67\pm2.06$ & $5.58\pm1.06$ & $1.76\pm0.54$ & $1.40\pm0.29$ & \multirow{5}{*}{$26.48$} \\
 &  & 0.1 & \bm{$39.17\pm0.65$} & $17.02\pm1.20$ & $17.04\pm1.50$ & $16.21\pm0.52$ & $1.48\pm0.23$ &  \\
 &  & 0.2 & \bm{$41.02\pm0.54$} & $18.36\pm0.56$ & $11.96\pm0.45$ & $22.18\pm0.31$ & $1.28\pm0.21$ &  \\
 &  & 0.5 & \bm{$41.42\pm3.71$} & $39.79\pm0.24$ & $32.13\pm3.68$ & $26.56\pm3.46$ & $1.91\pm0.49$ &  \\
 &  & 1 & \bm{$42.93\pm0.40$} & $35.59\pm1.79$ & $$ & $$ & $2.09\pm0.28$ &  \\
\cmidrule(lr{0.5em}){2-9} & \multirow{2}{*}{5} & 0.02 & \bm{$11.61\pm0.83$} & $3.04\pm1.07$ & $1.51\pm0.50$ & $1.0\pm0.00$ & $1.14\pm0.20$ & \multirow{5}{*}{$4.58$} \\
 &  & 0.1 & \bm{$18.73\pm0.78$} & $4.76\pm0.35$ & $5.13\pm0.15$ & $1.46\pm0.29$ & $1.19\pm0.27$ &  \\
 &  & 0.2 & \bm{$17.71\pm0.07$} & $4.30\pm1.15$ & $9.71\pm0.79$ & $1.26\pm0.38$ & $1.45\pm0.12$ &  \\
 &  & 0.5 & \bm{$19.13\pm0.44$} & $4.04\pm0.17$ & $15.64\pm0.49$ & $2.34\pm0.30$ & $1.32\pm0,29$ &  \\
 &  & 1 & \bm{$19.71\pm0.08$} & $3.27\pm1.15$ & $-$ & $-$ & $1.5\pm0.26$ &  \\
\cmidrule(lr{0.5em}){2-9} & \multirow{2}{*}{10} & 0.02 & \bm{$6.57\pm0.64$} & $1.50\pm0.25$ & $1.25\pm0.18$ & $1.0\pm0.00$ & $1.04\pm0.06$ & \multirow{5}{*}{$1.73$} \\
 &  & 0.1 & \bm{$5.59\pm0.48$} & $1.75\pm0.28$ & $1.97\pm0.27$ & $1.22\pm0.13$ & $1.26\pm0.15$ &  \\
 &  & 0.2 & \bm{$8.9\pm0.32$} & $1.86\pm0.26$ & $3.05\pm0.10$ & $1.52\pm0.16$ & $1.48\pm0.24$ &  \\
 &  & 0.5 & \bm{$8.21\pm0.12$} & $1.96\pm0.24$ & $6.28\pm0.52$ & $1.39\pm0.21$ & $1.36\pm0.17$ &  \\
 &  & 1 & \bm{$9.22\pm0.21$} & $1.85\pm0.23$ & $$ & $$ & $1.39\pm0.30$ &  \\
      \bottomrule
  \end{tabular}}
\end{table*}

\begin{table*}[htb]
  \caption{Comparison of different algorithms WCE-S2BL, KL-S2BL, WCE-S2BLK, KL-S2BLK and Bayesian-CRM (B-CRM) deterministic policy accuracy for FMNIST and CIFAR-10 with linear model setup and different qualities of logging policy ($\tau\in \{1,5,10,20\}$) and proportions of labeled data ($\rho\in\{0.02,0.1, 0.2,0.5,1\}$). }
  \label{tab:comparison main-linear-full}
  \centering
  \resizebox{\linewidth}{!}{
	\begin{tabular}{ccccccccc}
	\\
 \toprule
    	Dataset & $\tau$ & $\rho$ & \textbf{WCE-S2BL} & \textbf{KL-S2BL} & \textbf{WCE-S2BLK} & \textbf{KL-S2BLK} & \textbf{B-CRM} & \textbf{Logging Policy}\\
    	\midrule
    	\multirow{4}{*}{FMNIST} & \multirow{2}{*}{1} & 0.02 & $84.37\pm 0.14$ & $71.67\pm 0.26$ & $78.84\pm 0.05$ & $74.71\pm 0.06$ & $64.67\pm 1.44$
 & \multirow{5}{*}{\bm{$91.73$}} \\
    	 &  & 0.1 & $84.18\pm 0.00$ & $75.43\pm 0.04$ & $82.35\pm 0.05$ & $72.45\pm 0.01$ & $70.38\pm 0.09$ &  \\
      &  & 0.2 & $83.59\pm 0.18$ & $71.88\pm 0.31$ & $83.05\pm 0.06$ & $74.06\pm 0.00$ & $70.99\pm 0.32$ & \\      &  & 0.5 & $84.14\pm 0.20$ & $71.03\pm 0.13$ & $83.85\pm 0.00$ & $71.05\pm 1.79$ & $71.76\pm 0.03$
 & \\
      &  & 1 & $84.24\pm 0.07$ & $69.44\pm 1.20$ & $-$ & $-$ & $72.42\pm 0.01$
 & \\
      \cmidrule(lr{0.5em}){2-9}
      & \multirow{2}{*}{5} & 0.02 & \bm{$83.51\pm 0.01$} & $19.60\pm 0.42$ & $75.24\pm 2.89$ & $19.48\pm 0.33$ & $64.49\pm 01.04$
 & \multirow{5}{*}{$53.97$} \\
    	 &  & 0.1 & \bm{$83.99\pm 0.02$} & $36.33\pm 11.60$ & $80.11\pm 0.09$ & $29.55\pm 3.72$ & $70.21\pm 0.07$ &  \\
      &  & 0.2 & \bm{$83.91\pm 0.07$} & $54.83\pm 1.68$ & $82.69\pm 0.19$ & $51.02\pm 9.33$ & $71.14\pm 0.10$
 & \\ &  & 0.5 & \bm{$83.91\pm 0.01$} & $59.49\pm 0.61$ & $83.47\pm 0.02$ & $72.13\pm 0.24$ & $71.86\pm 0.14$
& \\
      &  & 1 & \bm{$83.62\pm 0.01$} & $73.11\pm 0.60$ & $-$ & $-$ & $72.33\pm 0.06$ & \\
      \cmidrule(lr{0.5em}){2-9}
      & \multirow{2}{*}{10} & 0.02 & \bm{$82.31\pm 0.07$} & $26.71\pm 2.18$ & $77.43\pm 0.13$ & $18.35\pm 7.06$ & $66.24\pm 00.03$
 & \multirow{5}{*}{$20.72$} \\
    	 &  & 0.1 & \bm{$82.30\pm 0.04$} & $56.51\pm 9.65$ & $77.59\pm 0.34$ & $47.93\pm 5.15$ & $70.33\pm 0.33$  & \\
      &  & 0.2 & \bm{$83.15\pm 0.09$} & $67.10\pm 5.17$ & $81.20\pm 0.12$ & $60.26\pm 0.88$ & $71.02\pm 0.30$ &  \\
      &  & 0.5 & \bm{$83.27\pm 0.01$} & $74.97\pm 0.17$ & $82.85\pm 0.10$ & $70.02\pm 1.41$ & $71.72\pm 0.01$ & \\
      &  & 1 & \bm{$83.00\pm 0.09$} & $73.92\pm 0.27$ & $-$ & $-$ & $72.25\pm 0.10$ & \\
      \cmidrule(lr{0.5em}){2-9}
    	 & \multirow{2}{*}{20} & 0.02 & $47.44\pm 2.83$ & $32.24\pm 4.95$ & $51.21\pm 2.52$ & $21.66\pm 1.89$ & \bm{$63.99\pm 1.01$}
 & \multirow{5}{*}{$10.54$} \\
    	 & & 0.1 & \bm{$75.10\pm 0.09$} & $69.22\pm 4.09$ & $75.02\pm 0.04$ & $59.04\pm 0.59$ & $68.43\pm 0.33$ & \\
      &  & 0.2 & $77.19\pm 0.02$ & $74.43\pm 0.78$ & \bm{$77.36\pm 0.02$} & $73.36\pm 1.51$ & $69.21\pm 0.24$ & \\      &  & 0.5 & $73.89\pm 0.00$ & \bm{$79.04\pm 0.17$} & $77.5\pm 0.17$ & $78.92\pm 0.04$ & $71.17\pm 0.05$ & \\
      &  & 1 & \bm{$78.51\pm 0.01$} & $74.36\pm 0.01$ & $-$ & $-$ & $71.74\pm 0.16$ & \\
      \midrule
\multirow{4}{*}{CIFAR-10} & \multirow{2}{*}{1} & 0.02 & \bm{$62.95\pm0.08$} & $28.29\pm11.35$ & $9.49\pm0.72$ & $10.02\pm0.02$ & $55.02\pm0.14$ & \multirow{5}{*}{$52.89$} \\
 &  & 0.1 & \bm{$62.97\pm0.27$} & $21.79\pm14.50$ & $60.89\pm0.13$ & $10.00\pm0.00$ & $56.59\pm0.26$ &  \\
 &  & 0.2 & \bm{$62.83\pm0.06$} & $26.29\pm6.92$ & $62.90\pm0.10$ & $14.08\pm1.58$ & $57.75\pm0.42$ &  \\
 &  & 0.5 & \bm{$63.49\pm0.07$} & $46.15\pm1.16$ & $62.89\pm0.07$ & $43.25\pm1.21$ & $58.81\pm0.03$ &  \\
 &  & 1 & \bm{$63.85\pm0.09$} & $44.07\pm0.79$ & $-$ & $-$ & $59.24\pm0.05$ &  \\
\cmidrule(lr{0.5em}){2-9} & \multirow{2}{*}{5} & 0.02 & \bm{$56.84\pm0.07$} & $15.03\pm6.04$ & $51.40\pm0.29$ & $13.74\pm2.47$ & $44.48\pm1.02$ & \multirow{5}{*}{$40.96$} \\
 &  & 0.1 & \bm{$57.47\pm0.28$} & $17.69\pm1.21$ & $55.85\pm0.19$ & $9.96\pm0.06$ & $50.88\pm0.32$ &  \\
 &  & 0.2 & \bm{$58.50\pm0.02$} & $18.74\pm1.58$ & $58.22\pm0.05$ & $13.26\pm3.54$ & $52.56\pm0.21$ &  \\
 &  & 0.5 & \bm{$59.47\pm0.09$} & $23.60\pm2.68$ & $60.03\pm0.03$ & $18.06\pm2.84$ & $53.44\pm0.28$ &  \\
 &  & 1 & \bm{$60.97\pm0.01$} & $34.35\pm2.59$ & $-$ & $-$ & $54.13\pm0.09$ &  \\
\cmidrule(lr{0.5em}){2-9} & \multirow{2}{*}{10} & 0.02 & \bm{$54.47\pm1.34$} & $11.60\pm1.13$ & $41.93\pm1.25$ & $10.08\pm0.12$ & $44.66\pm0.29$ & \multirow{5}{*}{$36.6$} \\
 &  & 0.1 & \bm{$55.47\pm0.29$} & $19.19\pm0.19$ & $54.84\pm0.02$ & $10.00\pm0.00$ & $50.76\pm0.25$ &  \\
 &  & 0.2 & \bm{$56.99\pm0.00$} & $22.83\pm0.46$ & $56.94\pm0.19$ & $13.69\pm2.69$ & $52.09\pm0.43$ &  \\
 &  & 0.5 & $60.27\pm0.08$ & $30.11\pm3.22$ & \bm{$60.77\pm0.00$} & $24.60\pm2.35$ & $53.19\pm0.42$ &  \\
 &  & 1 & \bm{$61.14\pm0.04$} & $40.54\pm0.48$ & $-$ & $-$ & $53.75\pm0.14$ &  \\
\cmidrule(lr{0.5em}){2-9} & \multirow{2}{*}{20} & 0.02 & \bm{$56.33\pm0.16
$} & $13.92\pm5.55
$ & $46.27\pm2.51
$ & $10.00\pm0.00
$ & $45.11\pm0.82
$ & \multirow{5}{*}{$41.63$} \\
 &  & 0.1 & \bm{$57.23\pm0.00$} & $20.79\pm0.03$ & $56.43\pm0.18$ & $13.92\pm0.52$ & $50.69\pm0.43$ &  \\
 &  & 0.2 & $57.87\pm0.11$ & $16.3\pm4.20$ & \bm{$57.90\pm0.27$} & $11.73\pm1.90$ & $51.88\pm0.22$ &  \\
 &  & 0.5 & $59.05\pm0.14$ & $24.16\pm0.67$ & \bm{$59.10\pm0.30$} & $19.23\pm0.37$ & $53.08\pm0.14$ &  \\
 &  & 1 & \bm{$61.76\pm0.16$} & $33.98\pm0.88$ & $-$ & $-$ & $53.51\pm0.16$ &  \\
      \bottomrule     
  \end{tabular}}
\end{table*}

\begin{table*}[htb]
{\blue
  \caption{Comparison of different algorithms WCE-S2BL, KL-S2BL, WCE-S2BLK, KL-S2BLK and Bayesian-CRM (B-CRM) deterministic policy accuracy for EMNIST and CIFAR-100 with linear model setup and different qualities of logging policy ($\tau\in \{1,5,10,20\}$) and proportions of labeled data ($\rho\in\{0.02,0.1, 0.2,0.5,1\}$). }
  \label{tab:comparison main-linear-EM-Cifar-100}
  \centering
  \resizebox{\linewidth}{!}{
	\begin{tabular}{ccccccccc}
	\\
 \toprule
    	Dataset & $\tau$ & $\rho$ & \textbf{WCE-S2BL} & \textbf{KL-S2BL} & \textbf{WCE-S2BLK} & \textbf{KL-S2BLK} & \textbf{B-CRM} & \textbf{Logging Policy}\\
    	\midrule     
\multirow{4}{*}{EMNIST} & \multirow{2}{*}{1} & 0.02 & \bm{$87.00\pm0.01$} & $77.18\pm0.37$ & $86.10\pm0.06$ & $52.52\pm0.68$ & $76.91\pm0.12$ & \multirow{5}{*}{$76.55$} \\
 &  & 0.1 & \bm{$87.52\pm0.00$} & $69.79\pm0.56$ & $86.92\pm0.07$ & $52.80\pm1.65$ & $80.84\pm0.07$ &  \\
 &  & 0.2 & \bm{$87.60\pm0.01$} & $79.83\pm0.50$ & $87.46\pm0.04$ & $76.11\pm0.69$ & $81.61\pm0.08$ &  \\
 &  & 0.5 & $87.69\pm0.04$ & $76.52\pm0.42$ & \bm{$87.71\pm0.03$} & $77.79\pm0.30$ & $82.02\pm0.09$ &  \\
 &  & 1 & \bm{$87.68\pm0.02$} & $80.83\pm0.73$ & $-$ & $-$ & $82.57\pm0.01$ &  \\
\cmidrule(lr{0.5em}){2-9} & \multirow{2}{*}{5} & 0.02 & \bm{$74.14\pm0.02$} & $33.86\pm0.38$ & $70.68\pm0.03$ & $15.14\pm4.23$ & $56.13\pm0.42$ & \multirow{5}{*}{$41.06$} \\
 &  & 0.1 & \bm{$82.10\pm2.21$} & $59.92\pm0.57$ & $62.42\pm0.33$ & $49.00\pm1.58$ & $62.79\pm0.20$ &  \\
 &  & 0.2 & \bm{$82.21\pm2.60$} & $69.39\pm0.37$ & $77.55\pm4.55$ & $51.28\pm6.94$ & $68.21\pm0.22$ &  \\
 &  & 0.5 & \bm{$84.91\pm2.87$} & $85.22\pm0.13$ & $76.38\pm3.00$ & $68.51\pm6.12$ & $73.12\pm0.17$ &  \\
 &  & 1 & $80.03\pm2.03$ & \bm{$86.81\pm0.05$} & $-$ & $-$ & $75.38\pm0.25$ &  \\
\cmidrule(lr{0.5em}){2-9} & \multirow{2}{*}{10} & 0.02 & \bm{$82.91\pm0.01$} & $33.54\pm1.24$ & $80.08\pm0.04$ & $9.67\pm0.38$ & $55.89\pm0.05$ & \multirow{5}{*}{$31.86$} \\
 &  & 0.1 & \bm{$82.95\pm0.03$} & $55.02\pm0.79$ & $82.40\pm0.01$ & $35.56\pm0.97$ & $65.94\pm0.33$ &  \\
 &  & 0.2 & $83.9\pm3.19$ & \bm{$84.27\pm0.07$} & $79.15\pm0.04$ & $83.16\pm0.31$ & $69.70\pm0.17$ &  \\
 &  & 0.5 & \bm{$88.01\pm0.15$} & $86.42\pm0.04$ & $85.28\pm2.68$ & $86.40\pm0.02$ & $73.43\pm0.20$ &  \\
 &  & 1 & \bm{$88.98\pm0.35$ }& $86.77\pm0.01$ & $-$ & $-$ & $75.18\pm0.19$ &  \\
\cmidrule(lr{0.5em}){2-9} & \multirow{2}{*}{20} & 0.02 & \bm{$82.17\pm0.04
$} & $23.34\pm0.40
$ & $78.25\pm0.16
$ & $22.71\pm2.07
$ & $54.02\pm0.93
$ & \multirow{5}{*}{$23.83$} \\
 &  & 0.1 & \bm{$87.72\pm0.14$} & $63.02\pm2.19$ & $86.89\pm0.01$ & $56.22\pm2.29$ & $67.20\pm0.46$ &  \\
 &  & 0.2 & \bm{$88.66\pm0.06$} & $82.93\pm0.25$ & $84.06\pm0.05$ & $82.21\pm0.32$ & $70.70\pm0.10$ &  \\
 &  & 0.5 & $89.66\pm0.09$ & $84.76\pm0.14$ & \bm{$89.78\pm0.05$} & $84.18\pm0.03$ & $73.94\pm0.12$ &  \\
 &  & 1 & \bm{$89.37\pm0.17$} & $80.00\pm0.10$ & $-$ & $-$ & $76.08\pm0.05$ &  \\
\midrule
\multirow{4}{*}{CIFAR-100} & \multirow{2}{*}{1} & 0.02 & \bm{$13.59\pm0.08$} & $6.81\pm2.94$ & $11.92\pm0.24$ & $2.59\pm2.24$ & $4.23\pm0.26$ & \multirow{5}{*}{$12.32$} \\
 &  & 0.1 & \bm{$13.65\pm0.05$} & $7.53\pm0.35$ & $12.60\pm0.06$ & $4.73\pm0.51$ & $8.29\pm0.05$ &  \\
 &  & 0.2 & \bm{$13.73\pm0.02$} & $9.07\pm0.53$ & $13.48\pm0.02$ & $5.73\pm0.27$ & $9.45\pm0.13$ &  \\
 &  & 0.5 & \bm{$13.70\pm0.02$} & $11.46\pm0.33$ & $13.56\pm0.10$ & $9.89\pm0.43$ & $10.94\pm0.03$ &  \\
 &  & 1 & \bm{$13.75\pm0.06$} & $11.99\pm0.99$ & $-$ & $-$ & $12.32\pm0.03$ &  \\
\cmidrule(lr{0.5em}){2-9} & \multirow{2}{*}{5} 
& 0.02 & \bm{$16.38\pm0.04$} & $1.85\pm1.21$ & $2.18\pm0.32$ & $1.0\pm0.00$ & $3.56\pm0.01$ & \multirow{5}{*}{$6.01$} \\
 &  & 0.1 & \bm{$16.14\pm0.05$} & $2.83\pm1.39$ & $15.30\pm0.07$ & $1.06\pm0.08$ & $6.68\pm0.09$ &  \\
 &  & 0.2 & \bm{$16.63\pm0.04$} & $4.67\pm0.70$ & $16.13\pm0.10$ & $1.90\pm0.60$ & $8.46\pm0.23$ &  \\
 &  & 0.5 & \bm{$16.62\pm0.04$} & $9.57\pm0.48$ & $16.49\pm0.06$ & $4.90\pm0.34$ & $9.83\pm0.09$ &  \\
 &  & 1 & \bm{$16.90\pm0.01$} & $12.24\pm0.29$ & $-$ & $-$ & $10.86\pm0.07$ &  \\
\cmidrule(lr{0.5em}){2-9} & \multirow{2}{*}{10} 
& 0.02 & \bm{$15.43\pm0.30$} & $1.25\pm0.36$ & $1.17\pm0.27$ & $1.0\pm0.00$ & $2.99\pm0.15$ & \multirow{5}{*}{$3.4$} \\
 &  & 0.1 & \bm{$15.90\pm0.11$} & $2.47\pm0.67$ & $10.81\pm0.03$ & $0.99\pm0.01$ & $6.01\pm0.11$ &  \\
 &  & 0.2 & \bm{$16.21\pm0.09$} & $2.84\pm0.18$ & $14.53\pm0.08$ & $1.44\pm0.20$ & $7.62\pm0.09$ &  \\
 &  & 0.5 & \bm{$16.71\pm0.06$} & $6.17\pm0.60$ & $16.80\pm0.04$ & $3.10\pm0.56$ & $8.77\pm0.06$ &  \\
 &  & 1 & \bm{$16.87\pm0.08$} & $6.04\pm1.18$ & $-$ & $-$ & $10.25\pm0.09$ &  \\
\cmidrule(lr{0.5em}){2-9} & \multirow{2}{*}{20} & 0.02 & \bm{$18.65\pm0.06
$} & $1.00\pm0.00
$ & $1.36\pm0.27
$ & $1.17\pm0.16
$ & $4.16\pm0.14
$ & \multirow{5}{*}{$3.22$} \\
 &  & 0.1 & \bm{$18.29\pm0.01$} & $1.60\pm0.42$ & $12.43\pm0.09$ & $1.07\pm0.07$ & $7.46\pm0.05$ &  \\
 &  & 0.2 & \bm{$16.99\pm0.10$} & $2.90\pm0.34$ & $5.31\pm0.23$ & $1.06\pm0.09$ & $8.80\pm0.09$ &  \\
 &  & 0.5 & \bm{$19.43\pm0.02$} & $4.38\pm0.21$ & $19.22\pm0.05$ & $2.19\pm0.70$ & $10.23\pm0.15$ &  \\
 &  & 1 & \bm{$20.36\pm0.11$} & $5.96\pm1.01$ & $-$ & $-$ & $11.43\pm0.01$ &  \\
      \bottomrule
  \end{tabular}}}
\end{table*}

\subsection{Propensity Score Truncation}
For improvement in regularization with KL divergence in the scenarios where the propensity scores in the logged missing-feedback dataset are zero, we use the propensity score truncation in \eqref{Eq: CRM regularize by KL} as follows:
\begin{align}\label{eq: kl final estim}
    &\hat{L}_{\mathrm{KL}}^{\nu}(\pi_\theta)\triangleq  \sum_{i=1}^k \frac{1}{m_{a_i}}\sum_{(x,a_i,p)\in S_{u}\cup S} \pi_\theta(a_i|x)\log\left(\pi_\theta(a_i|x)\right)-\pi_\theta(a_i|x)\log(\max(\nu,p)),
\end{align}
where $\nu \in [0,1]$ is the same truncation parameter for truncated IPS estimator in \eqref{Eq: truncated emp risk}. Note that in a case of $p_i=0$ for some sample $(x_i,a_i,p_i)\in S_{u}$ then we have $\hat{L}_{\mathrm{KL}}=-\infty$; hence considering $\nu$ in $\hat{L}_{\mathrm{KL}}^{\nu}$ will help to solve these cases.
\subsection{CIFAR-10 with pre-trained features} \label{sec:cifar_pretrained}
{\blue The linear experiments for CIFAR-10, Table~\ref{tab:comparison linear}, are trained based on a linear model for the logging policy and the learning policy, using pre-trained features as image representation. It is interesting to investigate the performance of a linear learning model on a deep logging policy.

\begin{table*}[htb]
  \caption{Comparison of different algorithms WCE-S2BL, KL-S2BL, WCE-S2BLK, KL-S2BLK and Bayesian-CRM (B-CRM) deterministic policy accuracy for  CIFAR-10 with linear model setup, deep logging policy and different qualities of logging policy ($\tau\in \{1,5,10,20\}$) and proportions of labeled data ($\rho\in\{0.02,0.1, 0.2,0.5,1\}$). The top and second-best performances are indicated by bold text and underlined bold text, respectively.}
  \label{tab:comparison Cifar-10-linear-full}
  \centering
  \resizebox{\linewidth}{!}{
	\begin{tabular}{ccccccccc}
	\\
 \toprule
    	Dataset & $\tau$ & $\rho$ & \textbf{WCE-S2BL} & \textbf{KL-S2BL} & \textbf{WCE-S2BLK} & \textbf{KL-S2BLK} & \textbf{B-CRM} & \textbf{Logging Policy}\\
    	\midrule
    \multirow{4}{*}{CIFAR-10} & \multirow{2}{*}{1} 
    & 0.02 & $\underline{\bm{39.39\pm 0.15}}$ & $37.21\pm 0.15$ & $30.56\pm 0.61$ & $30.08\pm 0.27$ & $19.00\pm 1.77$
 & \multirow{5}{*}{$\bm{79.77}$} \\
    	 &  & 0.1 & $\underline{\bm{40.18\pm 0.08}}$ & $37.74\pm 0.02$ & $35.76 \pm 0.04$ & $33.42 \pm 0.24$ & $27.72\pm 0.37$ &  \\
       &  & 0.2 & $\underline{\bm{40.66\pm 0.29}}$ & $37.88 \pm 0.58$ & $38.22 \pm 0.01$ & $35.70 \pm 0.25$ & $29.32\pm 0.35$ & \\  
    &  & 0.5 & $\underline{\bm{40.81\pm 0.08}}$ & $38.55\pm 0.14$ & $39.64\pm 0.14$ & $36.97\pm 0.06$ & $30.67 \pm 0.28$
 & \\
      &  & 1 & $\underline{\bm{40.77\pm 0.01}}$ & $38.07\pm 0.42$ & $-$ & $-$ & $31.32\pm 0.36$
 & \\
      \cmidrule(lr{0.5em}){2-9}
      & \multirow{2}{*}{5} 
      
      & 0.02 & $\underline{\bm{34.60\pm 0.06}}$ & $10.26\pm 0.37$ & $14.18\pm 5.92$ & $10.00\pm 0.00$ & $12.76\pm3.07$
 & \multirow{5}{*}{$\bm{53.97}$} \\
 
    	 &  & 0.1 & $\underline{\bm{39.91\pm 0.84}}$ & $10.90\pm 1.02$ & $35.08\pm 0.08$ & $10.40\pm 0.57$ & $24.50\pm 1.00$ &  \\
    
      &  & 0.2 & $\underline{\bm{40.15\pm 0.06}}$ & $11.58\pm 2.09$ & $37.50\pm 1.09$ & $11.52\pm 2.15$ & $27.70\pm 0.47$
 & \\ 
 &  & 0.5 & $\underline{\bm{39.90\pm  0.54 }}$& $31.61\pm  0.19$ & $38.67\pm 0.03$ & $20.51\pm 0.17$ & $29.50\pm 0.19$
& \\
      &  & 1 & $\underline{\bm{40.52\pm 0.07}}$ & $32.50\pm 0.84$ & $-$ & $-$ & $30.22\pm 0.81$ & \\
      \cmidrule(lr{0.5em}){2-9}
      & \multirow{2}{*}{10} 
      & 0.02 & $\underline{\bm{38.97\pm 0.03}}$ & $10.84\pm 1.18$ & $26.60\pm 0.89$ & $10.03\pm 0.05$ & $14.17\pm 2.87$
 & \multirow{5}{*}{$\bm{43.45}$} \\
    	 &  & 0.1 & $\underline{\bm{39.04\pm 0.02}}$ & $14.70\pm 5.24$ & $34.42\pm 0.10$ & $11.54\pm 2.18$ & $24.17\pm 3.25$  & \\
      &  & 0.2 &  $\underline{\bm{39.69\pm 0.05}}$ & $15.49\pm 2.23$ & $35.71\pm 0.73$ & $13.81\pm 2.74$ & $28.24\pm0.20$ &  \\
      &  & 0.5 & $\underline{\bm{39.57\pm 0.12}}$ & $28.52\pm 0.36$ & $38.53\pm 0.36$ & $20.80\pm 0.28$ & $29.78\pm 0.42$ & \\
      &  & 1 & $\underline{\bm{38.87\pm 0.23}}$ & $28.07\pm 0.92$ & $-$ & $-$ & $30.09\pm 0.47$ & \\
      \cmidrule(lr{0.5em}){2-9}
       & \multirow{2}{*}{20} & 0.02 & $\underline{\bm{17.03\pm 0.08}}$ & $11.1\pm 1.56$ & $16.39\pm 0.68$ & $10.01\pm 0.02$ & $10.25\pm 0.07$
 & \multirow{5}{*}{$\underline{\bm{20.72}}$} \\

    	 & & 0.1 & $\underline{\bm{20.46\pm 0.03}}$ & $11.54\pm 1.59$ & $18.87\pm 0.04$ & $11.46\pm 1.47$ & $15.28\pm 1.963$ & \\
      &  & 0.2 & $ \bm{22.06\pm 0.17}$ & $10.55\pm 0.65$ & $20.23\pm 0.02 $& $12.92\pm 1.03$ & $19.58\pm 1.10$ & \\

      &  & 0.5 & $\underline{\bm{23.30\pm 0.14}}$ & $14.35\pm 0.59$ & $21.11\pm 0.16$ & $16.99\pm 0.79$ & $\bm{24.77\pm 1.69}$ & \\
      &  & 1 & $\underline{\bm{25.35\pm 0.12}}$ & $23.21\pm 0.19$ & $-$ & $-$ & $\bm{25.40\pm 2.03}$ & \\
\bottomrule      
  \end{tabular}}
\end{table*}

We can observe in Table~\ref{tab:comparison Cifar-10-linear-full}, due to the fact that the complexity of logging policy as a deep model is more than a linear model, the linear CIFAR-10 model accuracy is worse than the logging policy. The reason behind this setting is that a simple linear model doesn't work well on the raw flattened image and pre-trained features inject unknown prior information into the input of the models. Therefore, the structure of the logging policy can affect the performance of the final learning policy. However, our algorithms, WCE-S2BL and KL-S2BL, outperform the baseline, B-CRM, Table \ref{tab:comparison linear}. 
\subsection{Real-World Experiments}\label{sec:real_world}
{\blue
We also carried out experiments on KuaiRec which is a dataset of human interactions with played videos in a mobile application. We adopt the setting introduced in
\citet{zhang2023uncertaintyaware} for our experiments. Our logging policy is a random sampler choosing between items available for each user with random probabilities with the constraint to achieve $70\%$ average cost. We assign random scores in $[0.001, 1.0]$ to each item that the user rated and normalize items with the same cost together and multiply the score of items with cost 1 by $0.7$ and other items by $0.3$ to get the $70\%$ average cost. We don't use explicit truncation for propensity scores in this dataset. For each user, we sample 5 items according to the logging policy to create the logged bandit dataset. \\
Because in KuaiRec, as a recommendation system dataset, each user (context) can have multiple preferred items (actions), the accuracy of the learned policy (proportion of correctly suggested items) can't give a complete evaluation of the model's performance. We use the empirical IPS, evaluated based on test dataset. \\
We train the models with batch-size 32 and an initial learning rate of 0.01 with a cosine annealing learning rate scheduler and use automatic hyper-parameter tuning for other hyper-parameters. We repeat each experiment 5 times and report the average and standard deviation of scores. Table \ref{tab:comparison kuairec} shows our results.
}

\begin{table}[H]

  \caption{Comparison of different algorithms WCE-S2BL, KL-S2BL, BanditNet empirical IPS for KuaiRec with different proportions of labeled data ($\rho\in\{0.02,0.1, 0.2,0.5,1\}$). }
  \label{tab:comparison kuairec}
  \centering
  \begin{center}
  \resizebox{0.7\linewidth}{!}{
	\begin{tabular}{ccccc}
	\\
 \toprule
    	Dataset & $\rho$ & \textbf{WCE-S2BL} & \textbf{KL-S2BL} & \textbf{BanditNet} \\
    	\midrule     
\multirow{5}{*}{KuaiRec} & 0.02 & $0.73\pm 0.33$ & \bm{$0.88\pm 0.28$} & $0.74\pm 0.43$ \\
& 0.1 & \bm{$0.73\pm 0.27$} & $0.69\pm 0.19$ & $0.58\pm 0.09$  \\
& 0.2 & \bm{$0.70\pm 0.13$} & $0.62\pm 0.19$ & $0.69\pm 0.34$ \\
& 0.5 & \bm{$0.76\pm 0.26$} & $0.72\pm0.20$ & $0.66\pm 0.17$\\
& 1.0 & \bm{$0.94\pm 0.20$} & $0.73\pm 0.12$ & $0.66\pm 0.23$ \\
      \bottomrule
  \end{tabular}}
  \end{center}
\end{table}
\clearpage
\subsection{Performance discussion  and analysis}\label{app: discussion}

Given that the same logged data were applied in both the linear and deep models, an apparent observation is the enhanced performance displayed by the deep model for all datasets, i.e., FMNIST, CIFAR-10, CIFAR-100 and EMNIST. The performance improvement also depends on the \emph{available portion of the logged known-feedback dataset}, denoted as $\rho$ and the \emph{quality of the initial logging policy}. In particular, both Tables~\ref{tab:comparison main-banditnet-full} and \ref{tab:comparison main-linear-full} demonstrate that when the logging policy is nearly uniform (i.e., Large $\tau$), superior performance is predominantly realized through WCE-S2BL and KL-S2BL. In addition, it is observed that in the majority of cases, when we have access to a relatively minor segment of the logged known-feedback dataset (e.g., $\rho=0.02$), the performance of WCE-S2BL is superior. This superior performance is particularly evident within the FMNIST and CIFAR-10 datasets for the deep model, where WCE-S2BL typically surpasses the performance of other proposed methods and B-CRM as the baseline. 

In the linear model (Tables \ref{tab:comparison main-linear-full}, \ref{tab:comparison main-linear-EM-Cifar-100}), the accuracy of WCE-S2BL remains high while the accuracy of the logging policy decreases and also keeps a significant gap with Bayesian-CRM model. In the deep setting, the same happens. Note that, in the linear model at $\tau=1$, wherein the performance of the logging policy exceeds $90\%$, there is an absence of algorithms demonstrating superior performance relative to the logging policy. It can be due to the complexity of feature space and the limitation of the linear model. The same phenomenon is also observed in CIFAR-10 for the linear model. In deep model setup, we observe that the WCE-S2BL for $\tau=1$ and FMNIST has better performance with respect to other proposed methods. We can also observe the performance improvement in KuaiRec dataset (Table~\ref{tab:comparison kuairec}) when using WCE-S2BL algorithm.

It is worthwhile to mention that for the logging policy close to uniform, our methods have better performance in both linear and deep models. 

Regarding the performance of WCE-S2BLK and KL-S2BLK with respect to WCE-S2BL and KL-S2BL, we can observe that in all cases, the logged missing-feedback dataset, can help us to achieve a better performance. This suggests that the inclusion of the logged missing-feedback dataset is beneficial for optimizing KL divergence (or reverse KL divergence), leading to a more accurate estimation and reduced variance of the IPS estimator. In particular, from Proposition~\ref{Prop: error analysis}, we expected that the error of estimators of KL divergence and reverse KL divergence would be reduced by using more data samples. Therefore, the logged missing-feedback dataset, can help to minimize the KL divergence and reverse KL divergence with a better estimation error.

From Table~\ref{tab:comparison main-banditnet-full} and Table~\ref{tab:comparison main-linear-full} we can also observe that as the number of samples decreases, WCE-S2BL keeps a more stable performance and its decrease in accuracy is much less than other methods and it's even negligible in many cases, such as $\tau=1$. This feature makes WCE-S2BL the best performer for $\rho=0.02$, when the proportion of the labeled data is smallest, for all settings and datasets. 

In experiments with deep models (as shown in Tables \ref{tab:comparison main-banditnet-full} and \ref{tab:comparison main-banditnet-full-emn-cifar-100}), WCE-S2BL achieves the highest performance in \textbf{47 out of 65} scenarios, with comparable performance in the remaining ones. For the linear model, this ratio is \textbf{70 out of 80} scenarios.





\subsection{Direct Approach: Q-learning approach}\label{App: direct approach}

Inspired by Pseudo-labeling approach in semi-supervised learning, we can propose Q-learning approach (cost-function estimation). In this approach, we first estimate the cost function using a logged known-feedback dataset. Using the model for cost function, we assign pseudo-feedback to the logged missing-feedback dataset. Then we train the final model via truncated IPS estimator using both logged known-feedback and Pseudo-feedback datasets.

For estimation of the cost function, we employed a logistic regression with a sigmoid activation function and a linear layer. Note that, in this scenario the feedback are binary. Second, we generate pseudo-feedback by applying the cost function estimator to the logged missing-feedback dataset. Finally, we train the truncated IPS estimator with both the logged known-feedback dataset and the pseudo-feedback dataset.

In Table~\ref{tab:comparison Q-learning}, we present the results (accuracy) of our algorithms (WCE-S2BL and KL-S2BL) and Q-learning under the EMNIST dataset with varying ratios of missing-feedback data to known-feedback data.

\begin{table}[H]
  \centering
  \caption{Comparison of different algorithms WCE-S2BL, KL-S2BL and Q-learning for EMNIST with linear model setup and different qualities of logging policy ($\tau\in \{1,5,10,20\}$) and proportions of labeled data ($\rho\in\{0.02,0.1, 0.2,0.5,1\}$). }
    \vspace{-5mm}
  \label{tab:comparison Q-learning}
  \resizebox{0.7\linewidth}{!}{
	\begin{tabular}{ccccccc}
	\\
 \toprule
    	Dataset & $\tau$ & $\rho$ & \textbf{WCE-S2BL} & \textbf{KL-S2BL} & \textbf{Q-Learning} & \textbf{Logging Policy}\\
    	\midrule     
\multirow{4}{*}{EMNIST} & \multirow{5}{*}{1} & 0.02 & \bm{$87.00\pm 0.01$} & $77.18\pm 0.37$ & $26.16\pm 1.30$ & \multirow{5}{*}{$76.55$} \\
 &  & 0.1 & \bm{$87.52\pm 0.00$} & $69.79\pm 0.56$ & $22.34\pm 0.48$ &  \\
 &  & 0.2 & \bm{$87.60\pm 0.01$} & $79.83\pm 0.50$ & $21.99\pm 0.93$ &  \\
 &  & 0.5 & \bm{$87.69\pm 0.04$} & $76.52\pm 0.42$ & $11.17\pm 0.25$ &  \\
 &  & 1.0 & \bm{$87.68\pm 0.02$} & $80.83\pm 0.73$ & $10.00\pm 0.00$ &  \\
    	\cmidrule(lr{0.5em}){2-7}     
 & \multirow{5}{*}{5} & 0.02 & \bm{$74.14\pm 0.02$} & $33.86\pm 0.38$ & $10.0\pm 0.00$ & \multirow{5}{*}{$41.06$} \\
 &  & 0.1 & \bm{$82.10\pm 2.21$} & $59.92\pm 0.57$ & $21.37\pm 4.35$ &  \\
 &  & 0.2 & \bm{$82.21\pm 2.60$} & $69.39\pm 0.37$ & $12.74\pm 3.87$ &  \\
 &  & 0.5 & $84.91\pm 2.87$ & \bm{$85.22\pm 0.13$} & $59.80\pm 5.12$ &  \\
 &  & 1.0 & $80.03\pm 2.03$ & \bm{$86.81\pm 0.05$} & $81.08\pm 7.16$ &  \\
     	\cmidrule(lr{0.5em}){2-7}     
 & \multirow{5}{*}{10} & 0.02 & \bm{$82.91\pm 0.01$} & $33.54\pm 1.24$ & $30.43\pm 4.50$ & \multirow{5}{*}{$31.86$} \\
 &  & 0.1 & \bm{$82.95\pm 0.03$} & $55.02\pm 0.79$ & $22.2\pm 8.80$ &  \\
 &  & 0.2 & $83.90\pm 3.19$ & \bm{$84.27\pm 0.07$} & $24.14\pm 10.54$ &  \\
 &  & 0.5 & \bm{$88.01\pm 0.15$} & $86.42\pm 0.04$ & $59.22\pm 0.59$ &  \\
 &  & 1.0 & \bm{$88.98\pm 0.35$} & $86.77\pm 0.01$ & $82.12\pm 3.56$ &  \\
     	\cmidrule(lr{0.5em}){2-7}    
 & \multirow{5}{*}{20} & 0.02 & \bm{$82.17\pm 0.04$} & $23.34\pm 0.40$ & $27.97\pm 2.03$ & \multirow{5}{*}{$23.83$} \\
 &  & 0.1 & \bm{$87.72\pm 0.14$} & $63.02\pm 2.19$ & $26.76\pm 0.18$ &  \\
 &  & 0.2 & \bm{$88.66\pm 0.06$} & $82.93\pm 0.25$ & $36.71\pm 4.00$ &  \\
 &  & 0.5 & \bm{$89.66\pm 0.09$} & $84.76\pm 0.14$ & $50.48\pm 3.67$ &  \\
 &  & 1.0 & \bm{$89.37\pm 0.17$} & $80.00\pm 0.10$ & $84.46\pm 3.17$ &  \\
 \bottomrule
  \end{tabular}}
\end{table}

As we can observe, the performance of Q-learning approach in EMNIST is worse than our algorithms, WCE-S2BL and KL-S2BL. Note that the Pseudo-feedback for logged missing-feedback samples can be different from true feedback (cost). Therefore, we have some noise in feedback and the (truncated) IPS estimator underperforms under noisy-feedback \cite{wang2017optimal}. This phenomena is also known as confirmation bias in semi-supervised learning scenario. It is interesting to explore other estimator which are robust to noise in feedback and can improve the Q-learning approach under both known-feedback and missing-feedback datasets.

\subsection{Code} 
We thank the authors of \citet{aouali2023exponential} for kindly sharing their code with us. The code is available at \url{https://gitlab.com/armin_gm/semi_logged_bandit_kl}.
All our experiments were run using 3 servers, each one with a GTX 3090 GPU
 and 32GB of RAM,
\end{document}

%% file: math-commands.tex
\usepackage{amsmath,amsfonts,bm}

\def\eqref#1{(\ref{#1})}

\def\1{\bm{1}}

\DeclareMathAlphabet{\mathsfit}{\encodingdefault}{\sfdefault}{m}{sl}
\SetMathAlphabet{\mathsfit}{bold}{\encodingdefault}{\sfdefault}{bx}{n}

